\documentclass[nohyperref]{article}

\usepackage{microtype}
\usepackage{graphicx}
\usepackage{subfigure}
\usepackage{booktabs} 

\usepackage{titletoc}
\usepackage[page,header]{appendix}

\newcommand{\alglinelabel}{%
  \addtocounter{ALC@line}{-1}
  \refstepcounter{ALC@line}
  \label
}
\PassOptionsToPackage{noend}{algorithmic}

\usepackage[accepted,nohyperref]{icml2023}
\usepackage{hyperref}
\definecolor{mydarkblue}{rgb}{0,0.08,0.45}
\hypersetup{ %
pdftitle={},
pdfauthor={},
pdfsubject={Proceedings of the International Conference on Machine Learning 2023},
pdfkeywords={},
pdfborder=0 0 0,
pdfpagemode=UseNone,
colorlinks=true,
linkcolor=mydarkblue,
citecolor=mydarkblue,
filecolor=mydarkblue,
urlcolor=mydarkblue,
}

\ifdefined\isaccepted \else
\hypersetup{pdfauthor={Anonymous Submission}}
\fi

\renewcommand{\algorithmiccomment}[1]{ \hfill $\triangleright$ {\color{blue} #1}}

\usepackage{amsmath}
\usepackage{amssymb}
\usepackage{mathtools}
\usepackage{amsthm}

\usepackage[capitalize,noabbrev]{cleveref}

\theoremstyle{plain}
\newtheorem{theorem}{Theorem}[section]

\newtheorem{lemma}[theorem]{Lemma}
\newtheorem{corollary}[theorem]{Corollary}
\theoremstyle{definition}
\newtheorem{definition}[theorem]{Definition}

\theoremstyle{remark}
\newtheorem{remark}[theorem]{Remark}

\usepackage[textsize=tiny]{todonotes}

\usepackage{nicefrac}
\usepackage{bbm}
\usepackage{comment}

\usepackage{enumitem}
\setitemize{leftmargin=*, nosep}
\setenumerate{leftmargin=*, nosep}

\usepackage{footnote}
\usepackage{multirow}
\usepackage{colortbl}

\newcommand{\jiatai}[1]{{\color{cyan}  [\text{Jiatai:} #1]}}
\newcommand{\daiyan}[1]{{\color{violet}  [\text{Yan:} #1]}}

\renewcommand{\tilde}{\widetilde}
\renewcommand{\hat}{\widehat}
\renewcommand{\bar}{\overline}
\renewcommand{\O}{\operatorname{\mathcal O}}
\newcommand{\Otil}{\operatorname{\tilde{\mathcal O}}}
\newcommand{\E}{\operatornamewithlimits{\mathbb{E}}}

\allowdisplaybreaks

\renewcommand{\algorithmicrequire}{\textbf{Input:}}
\renewcommand{\algorithmicensure}{\textbf{Output:}}

\Crefname{ALC@line}{Line}{Lines}
\Crefformat{equation}{Eq. #2(#1)#3}
\Crefrangeformat{equation}{Eqs. #3(#1)#4 to #5(#2)#6}
\Crefmultiformat{equation}{Eqs. #2(#1)#3}{ and #2(#1)#3}{, #2(#1)#3}{ and #2(#1)#3}
\Crefrangemultiformat{equation}{Eqs. #3(#1)#4 to #5(#2)#6}{ and #3(#1)#4 to #5(#2)#6}{, #3(#1)#4 to #5(#2)#6}{ and #3(#1)#4 to #5(#2)#6}
\Crefname{algocf}{Algorithm}{Algorithms}

\icmltitlerunning{Banker-OMD: A Universal Approach for Delayed Online Bandit Learning}

\begin{document}

\twocolumn[
\icmltitle{Banker Online Mirror Descent:\\A Universal Approach for Delayed Online Bandit Learning}



\icmlsetsymbol{equal}{*}

\begin{icmlauthorlist}
\icmlauthor{Jiatai Huang}{equal,iiis}
\icmlauthor{Yan Dai}{equal,iiis}
\icmlauthor{Longbo Huang}{iiis}
\end{icmlauthorlist}

\icmlaffiliation{iiis}{Institute for Interdisciplinary Information Sciences, Tsinghua University, Beijing, China}

\icmlcorrespondingauthor{Longbo Huang}{longbohuang@mail.tsinghua.edu.cn}

\icmlkeywords{Machine Learning, ICML}

\vskip 0.3in
]



\printAffiliationsAndNotice{\icmlEqualContribution} 
\begin{abstract}
  We propose Banker Online Mirror Descent (\texttt{Banker-OMD}), a novel framework generalizing the classical Online Mirror Descent (OMD) technique in the online learning literature. 
  The \texttt{Banker-OMD} framework almost completely decouples feedback delay handling and the task-specific OMD algorithm design, thus facilitating the design of new algorithms capable of efficiently and robustly handling feedback delays. Specifically, it offers a general methodology for achieving $\Otil(\sqrt{T} + \sqrt{D})$-style regret bounds in online bandit learning tasks with delayed feedback, where $T$ is the number of rounds and $D$ is the total feedback delay. We demonstrate the power of \texttt{Banker-OMD} by applications to two important bandit learning scenarios with delayed feedback, including delayed scale-free adversarial Multi-Armed Bandits (MAB) and delayed adversarial linear bandits. 
  \texttt{Banker-OMD} leads to the first delayed scale-free adversarial MAB algorithm achieving $\Otil(\sqrt{K}L(\sqrt T+\sqrt D))$ regret and the first delayed adversarial linear bandit algorithm achieving $\Otil(\text{poly}(n)(\sqrt{T} + \sqrt{D}))$ regret.
  As a corollary, the first application also implies $\Otil(\sqrt{KT}L)$ regret for non-delayed scale-free adversarial MABs, which is the first to match the $\Omega(\sqrt{KT}L)$ lower bound up to logarithmic factors and can be of independent interest.
\end{abstract}

\section{Introduction}
Multi-armed bandit (MAB) is a classical online learning problem with partial information feedback. In the MAB problem, an agent is given a set of arms and needs to choose one each time, after which the agent suffers a loss determined by the environment. The agent's objective is to minimize the expected difference between his/her total loss and the total loss of a fixed best arm, which is called the regret.

Among the many techniques successfully applied to MAB algorithm design, Online Mirror Descent (OMD) \citep{warmuth1997continuous} and Follow-The-Regularized-Leader (FTRL) \citep{gordon1999regret} have been proven to be powerful tools, especially in adversarial MAB settings --- where at each round $t$, an adversary arbitrarily picks a loss for each arm when the agent is making a decision. The OMD/FTRL perspective offers an alternative algorithmic form of classical adversarial MAB algorithms originating from the idea of exponential weighting, particularly the well-known EXP3 algorithm \citep{auer2002nonstochastic}. It also leads to the development of new MAB algorithms such as Tsallis-INF \citep{zimmert2019optimal} and BROAD-OMD \citep{wei2018more} by using different regularizers. At last, the OMD/FTRL perspective also generalizes to other online learning tasks, e.g., adversarial linear bandits \citep{abernethy2008competing,audibert2014regret} or Markov Decision Processes \citep{jin2020learning}.

However, in many important practical learning problems, action feedback may not arrive immediately after execution. For instance, in the web advertisement scenario \citep{li2010contextual}, after the server selects a set of ads and renders the page for a user, the signal about user reactions may arrive after the server picks ads for other incoming users. Similar situations can also happen in parallel computing \citep{chen2019task} where jobs need to be allocated to some work node before previous jobs finish or medical experiments \citep{wason2014comparison} where patients need to get treated before the information of previous treatment is gathered.

In such cases where feedback is arbitrarily delayed, the classical OMD approach can no longer be directly applied to guarantee similar performance.
Technically, this is because the classical OMD/FTRL framework heavily relies on techniques to rearrange potential terms into a telescoping sum so that terms from \textit{adjacent} rounds cancel. Thus, delayed feedback makes this real-time cancellation impossible.
While it is still possible to design algorithms for delayed MABs \citep{zimmert2020optimal,gyorgy2021adapting} or delayed MDPs \citep{jin2022near,dai2022follow} using OMD or similar ideas, one naturally asks the following question:

\textit{\textbf{Is there a general framework for transforming existing task-specific OMD algorithms into ones that can handle the ``delayed version'' of the corresponding tasks? 
}} 

In this paper, we provide a positive answer to this question by presenting a novel OMD framework, \texttt{Banker-OMD}.



\subsection{Our Contribution}
We first present \texttt{Banker-OMD}, a novel framework generalizing the classical OMD framework to handle feedback delays efficiently. Unlike classical OMD that plans a new action based only on the potential term of the previous \textit{single} round, \texttt{Banker-OMD} proposes a new perspective to take \textit{all} previous rounds into account, providing a robust rule to plan new actions when action feedback is delayed.
The \texttt{Banker-OMD} framework achieves $\Otil((\sqrt{T} + \sqrt{D})f(T))$ regret for various delayed bandit learning tasks, where $T$ is the number of rounds, $D$ is the total delay, and $f(T)$ is a task-specific factor irrelevant to delays.
More importantly, our framework keeps the flexibility of the classical OMD framework --- it depends on neither the choice of the regularizer nor the loss estimators and can be equipped with various OMD techniques (e.g., doubling tricks).

To demonstrate the power of \texttt{Banker-OMD}, we apply it to two important bandit learning scenarios with delayed feedback. The first one is delayed scale-free adversarial MAB (\Cref{sec:delayed scale-free MAB}): in addition to the usual delays (i.e., the feedback of $t$-th round can suffer a delay of $d_t$), the losses can fall into a general \textit{unknown} range $[-L,L]$ instead of the restrictive range $[0,1]$ \citep{zimmert2020optimal,thune2019nonstochastic}. As we will explain, such a generalization is very important in real-world scenarios such as advertisement recommendations. We design algorithms that can achieve $\Otil(\sqrt{K}L(\sqrt T+\sqrt D))$ regret. Additionally, a small-loss style bound \citep{neu2015first} is also provided for better data adaptivity.
Consequently, when applying our algorithm to the non-delayed scale-free adversarial MAB problem, our bound dominates the SOTA $\Otil(\sqrt{KL_2}+L_\infty \sqrt{KT})=\Otil(K\sqrt TL)$ \citep{putta2021scale} and matches the $\Omega(\sqrt{KT}L)$ lower bound \citep{auer2002nonstochastic} up to logarithmic factors, which can be of independent interest (see \Cref{remark:comparision with putta} for more discussions).

The second one is delayed adversarial linear bandits (\Cref{sec-application-linear}), where the feedback can be \textit{arbitrarily} delayed and the losses are adversarial, significantly strengthening the uniformly delayed setting \citep{ito2020delay} or stochastic linear bandit setting \citep{NEURIPS2019_56cb94cb,vernade2020linear}. In this setting, an \texttt{Banker-OMD}-based $\Otil(\text{poly}(n)(\sqrt T+\sqrt D))$-regret algorithm is yielded.

To our knowledge, we are the first to handle feedback delays in adversarial scale-free MABs and the first to allow \textit{arbitrary} (i.e., unrestricted and also unknown) delays in adversarial linear bandits, respectively.
For a better comparison, we also present an overview of our algorithms together with several related works in \Cref{tab:table-regret} (in the appendix).

\subsection{Related Work}
Due to space limitations, only the most relevant works are discussed. \Cref{sec:related work} gives a more thorough discussion.

For the vanilla delayed adversarial MABs (with $[0,1]$-bounded losses), 
\citet{bistritz2019online,thune2019nonstochastic,zimmert2020optimal} achieved $\Otil(\sqrt{D}+\sqrt{KT})$ regret, which is optimal compared to the $\Omega(\sqrt{KT}+\sqrt{D\log K})$ lower bound \citep{cesa2016delay} up to logs. However, our work \textit{does not} build upon any of them. In \Cref{sec:zimmert paper}, we discuss why we develop a new framework.

For scale-free learning, most works in this line consider full-information feedback \citep{orabona2018scale} or stochastic MABs \citep{hadiji2020adaptation}. The only upper bound for adversarial MABs is $\Otil(K\sqrt T L)$ \citep{putta2021scale}, which does not consider feedback delays and still has a $\sqrt K$ gap with the $\Omega(\sqrt{KT}L)$ lower bound \citep{auer2002nonstochastic}. See \Cref{remark:comparision with putta} for more discussions on this.

For delayed linear bandits, \citet{ito2020delay} studied adversarial linear bandits with \textit{known uniform} delays and achieved $\Otil(\sqrt{n(n+d)T})$ regret. All other works (see, e.g., \citep{NEURIPS2019_56cb94cb,vernade2020linear}) only consider stochastic losses. On the contrary, our algorithm can handle both adversarial losses and \textit{arbitrary and unknown} feedback delays.





\section{Problem Setup: Delayed Adversarial MAB}
\textbf{Notations.}
For $n\ge 1$, we denote the set $\left\{1,2,\ldots, n\right\}$ by $[n]$. We denote the probability simplex over $[n]$ by $\triangle^{[n]}$. We use $\mathbf 0$ to denote the zero vector. We use $\mathbf{1}_{i}$ to denote the one-hot vector with $1$ on the $i$-th coordinate, i.e., $(\mathbf 1_i)_j=\mathbbm{1}[i=j]$. We use $\Otil$ and $\tilde{\Theta}$ to ignore all logarithmic dependencies.
Let $f$ be a strictly convex function defined on some convex domain $A \subseteq \mathbb{R}^K$. For any $x,y \in A$, if $\nabla f(x)$ exists, denote the Bregman divergence between $y$ and $x$ induced by $f$ as
\begin{equation*}
D_f(y, x) \triangleq f(y) - f(x) - \langle \nabla f(x), y-x\rangle.
\end{equation*}
We use $f^\ast(y) \triangleq \sup_{x \in \mathbb{R}^K} \left\{ \langle y, x \rangle - f(x) \right\}$
to denote the Fenchel conjugate of $f$. At last, we use $\bar{f}$ to denote the restriction of function $f$ on $\triangle^{[K]}$, i.e., $\bar{f}(x)=f(x)$ if $x\in \triangle^{[K]}$ and $\bar{f}(x)=\infty$ otherwise.


This section focuses on delayed adversarial MABs \citep{zimmert2020optimal}, which we use to introduce our \texttt{Banker-OMD} framework. For the more general delayed scale-free adversarial MAB setting and delayed adversarial linear bandit setting, please refer to \Cref{sec:delayed scale-free MAB,sec-application-linear}. 

In a delayed adversarial MAB, there are $K \ge 2$ available actions (arms) and $T\ge 1$ rounds. For each round $t\in [T]$, there is a loss vector $l_t\in [0,1]^{K}$ obliviously determined by an adversary.
Meanwhile, there is also an oblivious delay $d_t$ associated with this round (also decided by the adversary).\footnote{For simplicity, we focus on the standard arm-independent delay model in the main text. Our framework is actually capable of the more general \textit{arm-dependent} delay model; see \Cref{apdx-arm-dependent-delays}.}
We define $D\triangleq \sum_{t=1}^T d_t$ following the convention \citep{bistritz2019online}. Losses and delays are all \emph{hidden} to the agent. 

At the beginning of $t$-th round, the agent chooses an action $A_t \in [K]$. The feedback $(t, l_{t, A_t})$ will be revealed to the agent at the end of the $t+d_t$-th round.
The agent can decide $A_t$ based on all historical actions,
all arrived observations,
and any private randomness.
Note that the agent does \textbf{not} have direct access to $d_t$'s, though it can infer the delay $d_t$ \textit{after} its feedback has arrived, i.e., after $d_t$ rounds.

The agent's objective is to minimize the difference between its total loss, i.e., $\sum_{t=1}^T l_{t, A_t}$ and the minimal possible total loss incurred by a single action. Formally speaking, we use pseudo-regret (also referred to as regret for convenience) as the performance metric for any MAB algorithm.
\begin{definition}
The pseudo-regret of an MAB algorithm is
\begin{equation*}
    \mathfrak{R}_T \triangleq \max_{i \in [K]} \E\left [\sum_{t=1}^T l_{t, A_t} - \sum_{t=1}^T l_{t, i}\right ],\label{eq:regret-def}
\end{equation*}
where the expectation is taken to both the algorithm's internal randomness and randomness from the environment. 
\end{definition}

\section{Vanilla OMD Framework}
\label{sec-classical-omd}
We review the vanilla Online Mirror Descent (OMD) framework for non-delayed MAB problems in \Cref{vanilla-omd}. It has been widely used in many bandit learning works (e.g., \citep{abernethy2008competing, abernethy2015fighting, audibert2014regret}).
Note that we use action scales $\sigma_t$ instead of learning rates $\eta_t$ for the ease of presentation, whose relationship is $\sigma_t=\eta_t^{-1}$.

\begin{algorithm}[htb]
\caption{Vanilla OMD for MAB}
\label{vanilla-omd}
\begin{algorithmic}[1]
\REQUIRE{Number of arms $K$, rounds $T$, Legendre regularizer $\Psi:\mathbb{R}_+^K\ \rightarrow \mathbb{R}$ (defined in \Cref{def:legendre}), initial action $x_1 \in \triangle^{[K]}$, action scales $\sigma_1, \ldots, \sigma_T$.}
\FOR{$t=1,2,\ldots,T$}
    \STATE Sample $A_t \in [K]$ according to $x_t$. Pull arm $A_t$.
    \STATE Receive $l_{t,A_t}$. Calculate $\tilde{l}_t \leftarrow \frac{l_{t,A_t}}{x_{t,A_t}}\mathbf{1}_{A_t}$. \alglinelabel{line-importance-sampling}
    \STATE Set $x_{t+1} \leftarrow \nabla\bar{\Psi}^*(\nabla\Psi(x_t) - \frac{1}{\sigma_t}\tilde{l}_t)$.
\ENDFOR
\end{algorithmic}
\end{algorithm}

Before presenting our framework, we first sketch the standard analysis of an OMD-based algorithm. By the OMD framework itself, we can conclude that (see, e.g., Theorem 28.4 by \citet{lattimore2020bandit}):
\begin{align}
&\quad \E[\langle l_t, x_t - y \rangle]=\E[\langle \tilde{l}_t, x_t - y \rangle]\qquad (\forall y\in \triangle^{[K]})\nonumber\\
&\le \E[\sigma_t D_\Psi(y, x_t) - \sigma_t D_\Psi(y, z_t) + \sigma_t D_\Psi(x_t, \tilde{z}_t)], \label{eq-omd-basic-mab}
\end{align}
where $z_t$ and $\tilde z_t$ are defined as
\begin{align}
z_t &= \nabla\bar{\Psi}^* \big (\nabla\Psi(x_t) - \sigma_t^{-1}\tilde{l}_t \big ), \nonumber\\
\tilde{z}_t &= \nabla\Psi^* \big (\nabla\Psi(x_t) - \sigma_t^{-1}\tilde{l}_t \big ). \label{eq-z-def}
\end{align}

\Cref{eq-omd-basic-mab} then gives the following by telescoping sums:
\begin{align}
\mathfrak R_T\le \sigma D_\Psi(y, x_1) + \sum_{i=1}^T\sigma\E[D_\Psi(x_t, \tilde{z}_t)]. \label{eq-omd-mab-cumu-fix-y}
\end{align}

When $\Psi$ satisfies certain conditions, one can find constants $C_1$ and $C_2$ such that (see, e.g., \citep{abernethy2015fighting}):
\begin{itemize}
\item $D_\Psi(y, x_1) \le C_1$ for any $y\in \triangle^{[K]}$, and
\item $\E[\sigma_t D_\Psi(x_t, \tilde{z}_t)] \le \frac{C_2}{\sigma_t}$ for any $x_t \in \triangle^{[K]}$ and $\sigma_t > 0$.
\end{itemize}

In such cases, we say $(\Psi, x_1)$ is \textit{$(C_1, C_2)$-regular}. Setting $\sigma_t\equiv \sqrt{\frac{C_2}{C_1}T}$ then gives the following according to \Cref{eq-omd-mab-cumu-fix-y}:
\begin{align}
\mathfrak R_T
& \le \sqrt{\frac{C_2}{C_1}T}C_1+T\left (\sqrt{\frac{C_2}{C_1}T}\right )^{-1} C_2 \nonumber\\
& =\O\big (\sqrt{C_1C_2T}\big ). \label{eq:vanilla OMD C_1C_2}
\end{align}

However, as noticed, the classical OMD framework heavily relies on the availability of $l_{t,A_t}$ at the end of round $t$ (based on which we compute $\tilde z_t$). 
Consequently, it's not easy to make an OMD-based algorithm capable of feedback delays.

To resolve this issue, the previous solution in the literature is to carefully design the loss estimators $\tilde l_{t}$ and/or the regularizer $\Psi$ (see, e.g., \citep{zimmert2020optimal,putta2021scale}).
However, such an idea limits the use of OMD to complicated applications like scal-free delayed adversarial MABs or delayed adversarial linear bandits.
We, on the other hand, design an enhanced framework that can easily handle feedback delays, leaving the regularizers and estimators to the user, as we will see in \Cref{sec:delayed scale-free MAB,sec-application-linear}.

\section{\texttt{Banker-OMD} for Delayed Bandit Feedback}
\label{sec-banker-omd}
We are ready to introduce our \texttt{Banker-OMD} framework. We begin by inspecting the classical \Cref{eq-omd-basic-mab} for the vanilla OMD framework, with a new set of terminologies for the terms This gives a better understanding of our framework. 

\subsection{Terminologies in Regret Decomposition}
To begin with, we sum \Cref{eq-omd-basic-mab} over $t=1,2,\ldots , T$ for an upper bound of $\mathfrak{R}_T$ and rename the terms as follows:
\begin{align}
\mathfrak{R}_T & \le \sum_{t=1}^T \underbrace{\E[\sigma_t D_\Psi(y, x_t)]}_{\text{withdrawal}} - \sum_{t=1}^T \underbrace{\E[\sigma_t D_\Psi(y, z_t)]}_{\text{saving}} \nonumber \\
&\quad + \sum_{t=1}^T \underbrace{\E[\sigma_t D_\Psi(x_t, \tilde{z}_t)]}_{\text{immediate cost}} \label{eq-omd-basic-mab-sum},
\end{align}
where $z_t$ and $\tilde z_t$ are defined in \Cref{eq-z-def}.


We now explain the idea behind the names in \Cref{eq-omd-basic-mab-sum}. Let us first look at the third term $\sigma_t D_\Psi(x_t, \tilde{z}_t)$. 
As mentioned in \Cref{sec-classical-omd}, one can bound $\E[\sigma_t D_\Psi(x_t, \tilde{z}_t) \mid  \mathcal{F}_{t-1}]\le \frac{C_2}{\sigma_t}=\O(\sigma_t^{-1})$ when $(\Psi,x_0)$ is chosen to be ``good enough.'' In other words, they can be bounded without affecting the first two terms. Hence, we refer to them as ``immediate costs'' as they immediately contribute to $\mathfrak R_T$ once $\sigma_t$'s are decided.

Then consider {\small $\sum_{t=1}^T \E[\sigma_t D_\Psi(y, x_t)] - \sum_{t=1}^T \E[\sigma_t D_\Psi(y, z_t)]$}. 
If we choose a constant scale $\sigma_t=\sigma$ and pick $x_{t+1} \leftarrow z_t$ (as in \Cref{sec-classical-omd}), it becomes a telescoping sum bounded by $\sigma D_\Psi(y, x_1) - \sigma D_\Psi(y, z_s)$. 
Therefore, we can view it in a step-by-step manner: In round $t$, we manage to reduce some regret by incurring a negative term $-\sigma D_\Psi(y,z_t)$. However, we will immediately use this term to pay the new regret $\sigma D_\Psi(y,x_{t+1})$.
After that, we compute $z_{t+1}$ and introduce a new negative term $-\sigma D_\Psi(y,z_{t+1})$ for subsequent rounds.
Thus, the $-\sigma D_\Psi(y,z_t)$ terms can be viewed as bank ``savings,'' which ensures that one can ``withdraw'' $\sigma D_\Psi(y, x_{t+1})$ to pay the cost of $x_{t+1}$ in the future, without overdrafting.


\subsection{Beyond Telescoping: Utilizing Multiple Savings}

We generalize the banker idea above to tackle the challenge of potentially absent feedback:
We use \textit{more than one} ``saving terms'' accumulated in previous rounds to make up the current ``withdrawal'' $\tilde \sigma D_\Psi(y,x)$ -- in contrast to vanilla OMD where only a \textit{single} saving is used.

Formally, we focus on a specific round $t\in [T]$. We aim to design the action $x_t$ and the action scale $\sigma_t$ wisely so that the withdrawal $\sigma_t D_\Psi(y,x_t)$ is covered by several previous savings from rounds $1\le t_1<t_2<\cdots<t_h<t$, namely $\sigma_{t_1}D_\Psi(y,z_{t_1})$, $\sigma_{t_2}D_\Psi(y,z_{t_2})$, ..., $\sigma_{t_h}D_\Psi(y,z_{t_h})$.\footnote{The vanilla OMD covered in the previous section therefore reduces to the special case of $h=1$ and $t_1=t-1$.}
In this section, we treat $t_1,t_2,\ldots,t_h$ as some given sequence -- its decision rule is the focus of the next section.

In this case, we have \Cref{lemma-banker-convex-comb}, where the choice of $x$ in \Cref{eq:convex-com} is analog to $z_t$ used in vanilla OMD (\textit{c.f.} \Cref{eq-z-def}).
\begin{lemma}
\label{lemma-banker-convex-comb}
For any $z_1,\ldots, z_h \in \triangle^{[K]}$, $\sigma_1,\ldots,\sigma_h > 0$ and Legendre convex $\Psi: \mathbb{R}_+^K \rightarrow \mathbb{R}$, let $\tilde \sigma = \sum_{i=1}^h \sigma_i$ and
\begin{equation}
    x = \nabla\bar{\Psi}^*\left (\sum_{i=1}^h \frac{\sigma_i}{\tilde \sigma}\nabla\Psi(z_i)\right ), \label{eq:convex-com}
\end{equation}
then we have
\begin{equation*}
    \sigma D_\Psi(y, x) \le \sum_{i=1}^m \sigma_i D_\Psi(y, z_i),\quad \forall y \in \triangle^{[K]}.
\end{equation*}
\end{lemma}
In other words, once we set $\sigma_t$ as $\sum_{i=1}^h \sigma_{t_i}$ and construct $x_t$ from \Cref{eq:convex-com}, we automatically ensure $\sigma_t D_\Psi(y,x_t)\le \sum_{i=1}^h \sigma_{t_i} D_\Psi(y,z_{t_i})$ -- i.e., the $t$-th withdrawal is covered by the savings from rounds $t_1,t_2,\ldots,t_h$ -- as desired.

\subsection{``Over-Drafting'' in Case of Saving Shortage}
In the presence of delayed feedback, the total accumulated savings might not be always enough. For example, there is no feedback in the first few rounds, but we still need to decide $x_t$ in real time.
To cover the remaining part of the withdrawal, we propose to ``over-draft'' some more savings via investing on a \textit{default action} $x_0\in \triangle^{[K]}$. 
Specifically, if we want to pick an action scale $\sigma_t>\sum_{i=1}^h \sigma_{t_i}$. Let the shortage be $b_t=\sigma_t - \sum_{i=1}^h \sigma_{t_i}$. We then add and subtract $\E[b_t] D_\Psi(y,x_0)$ to the RHS of \Cref{eq-omd-basic-mab-sum}, which gives
\begin{align}
&\mathfrak{R}_T \le \sum_{t=1}^T \underbrace{\E[\sigma_t D_\Psi(y, x_t)]}_{\text{withdrawal}} - \sum_{t=1}^T \underbrace{\E[\sigma_t D_\Psi(y, z_t)]}_{\text{saving}} \nonumber \\
& \qquad\,\, -\underbrace{\E[b_t] D_\Psi(y,x_0)}_{\text{imaginary saving}} + \sum_{t=1}^T \underbrace{\E[\sigma_t D_\Psi(x_t, \tilde{z}_t)]}_{\text{immediate cost}} \nonumber
 \\
 &\qquad\,\,  + \underbrace{\E[b_t] D_\Psi(y,x_0)}_{\text{investment cost}}.\label{eq:regret decomposition after investment}
\end{align}

As desired, we can use the negative term $- b_t D_\Psi(y,x_0)$ as an ``imaginary'' saving -- in addition to the $h$ ``actual'' savings $\sum_{i=1}^h \sigma_{t_i} D_\Psi(y,z_{t_i})$ -- to apply \Cref{lemma-banker-convex-comb}. This allows us to pay the withdrawal due to the following action:
\begin{equation}
x_t = \nabla\bar{\Psi}^*\left (\sum_{i=1}^h \frac{\sigma_{t_i}}{\sigma_t}\nabla\Psi(z_{t_i}) + \frac{b_t}{\sigma_t}\nabla\Psi(x_0)\right ),\label{eq:enhanced decision rule of x}
\end{equation}
which exactly uses up the $h$ actual savings $\sigma_{t_i}D_\Psi(y,z_{t_i})$ together with the imaginary one $-b_tD_\Psi(y,x_0)$.

We call $b_t$ an ``investment'': When the total savings are insufficient for a new action, we ``invest'' in some $x_0$ to make up the difference and proceed. As illustrated in \Cref{eq:regret decomposition after investment}, such investments are not for free. They each introduces an \textit{investment cost} $\E[b_t] D_\Psi(y,x_0)$ to our total regret.

Intuitively, if for each $t\in [T]$ we can assign a sequence of previous savings $\{(\sigma_{t,i},z_{t,i})\}_{i=1}^{h_t}$ together with an investment $b_t$, we can set $\sigma_t=b_t+\sum_{i=1}^{h_t} \sigma_{t,i}$, determine $x_t$ according to \Cref{eq:enhanced decision rule of x}, and yield the following regret bound:
\begin{align}
\mathfrak R_T
&\lesssim\underbrace{\E[B_T]\cdot D_\Psi(y,x_0)}_{\text{total investment cost}}+\sum_{t=1}^T \underbrace{\E[\sigma_t D_\Psi(x_t,\tilde z_t)]}_{\text{immediate cost}}, \label{eq:informal decomposition of regret}
\end{align}
where $B_T\triangleq \sum_{t=1}^T b_t$ is called the \textit{total investment}.

\begin{algorithm}[!t]
\caption{\texttt{Banker-OMD} Framework}
\label{banker-omd}
\begin{algorithmic}[1]
\REQUIRE{Number of arms $K$, regularizer and default investment $(\Psi,x_0)$, subroutine to pick action scales $S$}
\ENSURE{A sequence of actions $A_1, A_2,\ldots \in [K]$}
\STATE Initialize $B_0\leftarrow 0$ \COMMENT{maintain the total investment}
\FOR{$t=1,2,\ldots,T$}
    \STATE $a_t \leftarrow \mathrm{missing}$. \COMMENT{whether feedback of $t$ has arrived}
    \STATE $\sigma_t \leftarrow S(t, \{(A_s,a_s,l_{s,A_s},\sigma_s,v_s)\}_{s<t})$. \COMMENT{a user-specified rule of deciding the action scale $\sigma_t$}
    \STATE $v_t \leftarrow \sigma_t$. \COMMENT{coefficient of the saving term $D_\Psi(y, z_t)$} \alglinelabel{line:initialize v in Banker}
    \STATE Initialize $b_t\gets \sigma_t$. For all $s<t$ such that $a_s=\text{arrived}$, set $\sigma_{t,s}=\min\{v_s,b_t\}$ and $b_t\gets b_t-\sigma_{t,s}$.\COMMENT{determine the amount of ``investment'' and ``savings'' to form $x_t$ by minimizing $b_t$ while ensuring \Cref{eq:conditions of b and m}}\alglinelabel{line:deciding b and m}
    \STATE $B_t \leftarrow B_{t-1} + b_t$. \COMMENT{update total investment $B_t$}
    \FOR{$s=1,2,\ldots,t-1$}
        \STATE $v_s \leftarrow v_s - \sigma_{t,s}$. \COMMENT{spend $\sigma_{t,s}$ units of savings}\alglinelabel{line:spend savings in banker}
    \ENDFOR
    
    \STATE $x_t \leftarrow \nabla\bar{\Psi}^*(\frac{1}{\sigma_t}\sum_{s=1}^{t-1}\sigma_{t,s}\nabla\Psi(z_s)+\frac{b_t}{\sigma_t}\nabla\Psi(x_0))$.\label{line:xt-eq} \COMMENT{decide $x_t$ according to \Cref{lemma-banker-convex-comb} and \Cref{eq:enhanced decision rule of x}} \alglinelabel{line:decide x_t in banker}
    \STATE Sample $A_t \in [K]$ according to $x_t$, pull arm $A_t$.
    \FOR{upon receiving each new feedback $(s, l_{s,A_s})$}
        \STATE $\tilde{l}_s \leftarrow \frac{l_{s,A_s}}{x_{s,A_s}}\mathbf{1}_{A_s}$. $z_s \leftarrow \nabla\bar{\Psi}^*(\nabla\Psi(x_s) - \frac{1}{\sigma_s}\tilde{l}_s)$.
        \STATE $a_s \leftarrow\mathrm{arrived}$. \COMMENT{saving $-\sigma_s D_\Psi(y, x_s)$ available}
    \ENDFOR
\ENDFOR
\end{algorithmic}
\end{algorithm}

\subsection{Formal Framework: \texttt{Banker-OMD}}
In this section, we focus on assigning $\{(\sigma_{t,i},z_{t,i})\}_{i=1}^{h_t}$ and $b_t$ to every round $1\le t\le T$ so that \Cref{eq:informal decomposition of regret} is ensured.

For each round $s<t$, we denote its corresponding ``remaining savings'' as $v_s$, i.e., the coefficient before $D_\Psi(y,z_s)$ in the current regret upper bound is $-v_s$.
Thus, $v_s$ is initialized to $0$ and becomes $\sigma_s$ once the $s$-th feedback arrives (so $z_s$ is revealed and the saving $\sigma_s D_\Psi(y,z_s)$ becomes available).
As sketched in the previous section, we must decompose each action scale $\sigma_t$ into an investment $b_t$ together with the sum of several previous savings -- if we use $\sigma_{t,s}$ units of saving from round $s$ when constructing $\sigma_t$, then we need
\begin{equation}\label{eq:conditions of b and m}
b_t+\sum_{s=1}^{t-1}\sigma_{t,s} = \sigma_t,\quad \text{and}\quad \sigma_{t,s} \le v_s,
\end{equation}
As $b_t$ causes extra investment cost, we shall minimize $b_t$ while ensuring \Cref{eq:conditions of b and m} -- in our framework, we only allocate $b_t$ if using up all previous $v_s$'s is still insufficient.

We formally present our framework in \Cref{banker-omd}, which we call \texttt{Banker-OMD}.
As we sketched in the previous section, we shall expect its regret to be controlled by the total investment $B_T$, the action scales $\sigma_t$, and the remaining savings $v_t$ -- which is formalized in the following theorem.

\begin{theorem}[\texttt{Banker-OMD} Regret Bound]
\label{thm-banker-omd}
At the end of any round $T$, for any $y \in \triangle^{[K]}$, we have
{\begin{align}
\sum_{t=1}^T \langle \tilde{l}_t, x_t - y\rangle & \le \underbrace{B_T \cdot D_\Psi(y, x_0)}_{\text{total investment cost}} + \underbrace{\sum_{t=1}^T \sigma_t D_\Psi(x_t, \tilde{z}_t)}_{\text{total immediate costs}} \nonumber\\
&\quad - \underbrace{\sum_{t=1}^T v_t D_\Psi(y, z_t)}_{\text{remaining savings}}, \label{eq-thm-banker-omd}
\end{align}}
where $B_T,\tilde{l}_1,\ldots,\tilde{l}_t,x_1,\ldots,x_t, z_1,\ldots,z_t$ are the variable values produced by \Cref{banker-omd} at the end of round $T$ and $\tilde z_t$ is defined as $\tilde{z}_t = \nabla\Psi^*(\nabla\Psi(x_t) - \frac{1}{\sigma_t}\tilde{l}_t)$.
\end{theorem}

\Cref{thm-banker-omd} offers a regret bound for \texttt{Banker-OMD} under general parameter selection and serves as the basis for achieving $\Otil(\sqrt{T} + \sqrt{D})$-style regret bounds by choosing proper configurations (see \Cref{thm-banker-omd-B} for an easy example). 
The proof of \Cref{thm-banker-omd} generally follows from our informal intuition, i.e., inductively applying \Cref{lemma-banker-convex-comb}. A formal version can be found in \Cref{sec:appendix analysis of banker}.

Imitating \Cref{eq:vanilla OMD C_1C_2} for vanilla OMD, we give the following corollary by assuming $(C_1,C_2)$-regularity in \Cref{thm-banker-omd}.
\begin{corollary}
\label{corollary-adversarial}
If $(\Psi, x_0)$ is $(C_1,C_2)$-regular, i.e., $D_\Psi(y, x_0)\leq C_1$ for all $y \in \triangle^{[K]}$ and $\E[\sigma_t D_\Psi(x_t, \tilde{z}_t) \mid  \mathcal{F}_{t-1}]\leq \frac{C_2}{\sigma_t}$ for all $t\in [T]$, then the \texttt{Banker-OMD} framework presented in \Cref{banker-omd} ensures
\begin{equation*}
    \mathfrak{R}_T \le C_1\cdot\E\left[B_T\right] + C_2\cdot\E \left[\sum_{t=1}^T \sigma_t^{-1} \right].
\end{equation*}
\end{corollary}

\begin{remark}
If there is no delay and we set $\sigma_t\equiv \sigma$, then \texttt{Banker-OMD} coincides with the vanilla OMD with the same $\sigma$ -- which is indeed our inspiration for \texttt{Banker-OMD}.
\end{remark}

\subsection{Tuning the Action Scales Painlessly}
One may notice that in \Cref{banker-omd}, the regularizer $\Psi$ and the action scales $\sigma_t$ remain unspecified.
In fact, our \texttt{Banker-OMD} framework inherits the flexibility of vanilla OMD -- the users can arbitrarily pick $\Psi$ and $\sigma_t$ as they want.
For example, $\Psi$ can be any popular regularizer including negative entropy, log-barrier \citep{wei2018more}, and Tsallis entropy \citep{abernethy2015fighting}. 
The action scale $\sigma_t$ can also be freely decided, e.g., they may depend on the current round index $t$ and the statistics of experienced delays.\footnote{As a side note, we \textit{do not} require $\sigma_t$'s to be monotone -- even more flexible than the vanilla OMD framework.}

In this section, we exemplify a possible tuning of $\sigma_t$ which automatically ensures $\Otil(\sqrt{T} + \sqrt{D})$-style regret bounds in delayed bandit learning problems. 
In the following, we assume $(\Psi, x_0)$ to be $(C_1, C_2)$-regular, which solely depends on the the choice of $(\Psi,x_0)$ picked by the user.


Let $\mathfrak{d}_t$ denote the number of feedback that has not arrived at the beginning of time $t$. We define $\mathfrak{D}_t = \sum_{s=1}^t \mathfrak{d}_s$ as the cumulative experienced delay up to time $t$ -- both $\mathfrak d_t$ and $\mathfrak D_t$ can be easily maintained at run-time.
Since $\sum_{t=1}^T \frac{1}{\sqrt t}=\O(\sqrt T)$ and $\sum_{t=1}^T \frac{\mathfrak{d}_t}{\sqrt{\mathfrak{D}_t}} = \O(\sqrt{\mathfrak{D}_T}) = \O(\sqrt D)$ (\Cref{lem:summation lemma}), one only needs to tune the action scales $\sigma_t$ as
\begin{equation*}
\sigma_t=\left (\frac{1}{\sqrt{t}} + \frac{\mathfrak{d}_t}{\sqrt{\mathfrak{D}_t}}\right )^{-1},\quad \forall t\ge 1
\end{equation*}
to upper bound the total immediate costs as
\begin{equation*}
C_2\sum_{t=1}^T \sigma_t^{-1}=\O\left (C_2(\sqrt{T} + \sqrt{D})\right ).
\end{equation*}

For the total investment part, we make the following observation, which follows from a direct analysis of \Cref{line:deciding b and m}:
\begin{lemma}
\label{lemma-banker-omd-B}
Let $T_0$ be the largest number such that $b_{T_0}>0$. Then, at the end of time $T$, we can decompose $B_T$ as
\begin{equation*}
B_T = B_{T_0} = \sigma_{T_0} + \sum_{t=1}^{T_0 - 1}\mathbbm{1}[ t + d_t \ge T_0]\sigma_t.
\end{equation*}
\end{lemma}
Suppose that at the beginning of round $T_0$, there are still $m$ feedback on the way, each corresponding to rounds $t_1 < \cdots < t_m$. 
Since these $m$ feedback contributes at least $\binom{m+1}{2}$ to the total delay $D$, we have $m = \O(\sqrt{D})$. Thus, under the action scale $\sigma_t=\left (\frac{1}{\sqrt t}+\frac{\mathfrak d_t}{\sqrt{\mathfrak D_t}}\right )^{-1}$, we have
\begin{align*}
    B_T &= \sigma_{T_0} + \sum_{i=1}^m\sigma_{t_i} \le \sqrt{t_1} + \sum_{i=2}^m \frac{\sqrt{\mathfrak{D}_{t_i}}}{\mathfrak{d}_{t_i}} + \frac{\sqrt{\mathfrak{D}_{T_0}}}{\mathfrak{d}_{T_0}} \\
    &\le \sqrt{T} + \sum_{i=1}^m \frac{\sqrt{D}}{i} = \O\big (\sqrt{T} + \sqrt{D}\log D\big ).
\end{align*}

Therefore, we obtain the following theorem:

\begin{theorem}[Main Theorem: \texttt{Banker-OMD} with Delays]
\label{thm-banker-omd-B}
When setting $\sigma_t$ as follows in \Cref{banker-omd}:
\begin{equation*}
\sigma_t = \left (\frac{1}{\sqrt{t}} + \mathfrak{d}_t\sqrt{\frac{\ln( \mathfrak{D}_t+1)}{\mathfrak{D}_t}}\right )^{-1},\quad \forall 1\le t\le T,
\end{equation*}
we have $B
_T\le \O(\sqrt{T} + \sqrt{D \log D})$ at the end of round $T$. Moreover, if in addition $(\Psi, x_0)$ is $(C_1, C_2)$-regular, setting
\begin{equation*}
\sigma_t = \sqrt{\frac{C_2}{C_1}}\left (\frac{1}{\sqrt{t}} + \mathfrak{d}_t\sqrt{\frac{\ln( \mathfrak{D}_t+1)}{\mathfrak{D}_t}}\right )^{-1},\quad \forall 1\le t\le T
\end{equation*}
guarantees 
\begin{equation*}
    \mathfrak R_T=\O\left (\sqrt{C_1 C_2} (\sqrt{T} + \sqrt{D\log D})\right ).
\end{equation*}
\end{theorem}
\begin{remark}
This theorem only provides an easy way of getting $\Otil(\sqrt{T}+\sqrt {D})$ regret. As we said, the action scales are up to the user -- as we will demonstrate shortly, a different and more sophisticated design of $\sigma_t$ is used in \Cref{sec:delayed scale-free MAB}.
\end{remark}

\section{Application 0: Delayed Adversarial MAB}

For a sanity check and a simple yet illustrative example, we consider the standard delayed adversarial MAB setting.
Inspired by the \texttt{Tsallis-INF} algorithm for adversarial MABs \citep{zimmert2019optimal}, we use the $\nicefrac 12$-Tsallis entropy regularizer $\Psi(x)=-\sum_{i=1}^K2\sqrt{x_i}$.
To ensure regularity, we pick $x_0=(\frac 1K,\frac 1K,\ldots,\frac 1K)$, which ensures $(\Psi,x_0)$ is $(\O(\sqrt{K}), \O(\sqrt{K}))$-regular \citep{abernethy2015fighting}.
At last, as suggested by \Cref{thm-banker-omd-B}, we determine $\sigma_t^{-1}$ as $\frac{1}{\sqrt t}+\mathfrak d_t \sqrt{\frac{\ln (\mathfrak D_t+1)}{\mathfrak D_t}}$.
Plugging all these configurations into \Cref{banker-omd}, we get the algorithm \texttt{Banker-TINF}, whose performance guarantee is stated in \Cref{thm-banker-tinf}. 

\begin{corollary}
\label{thm-banker-tinf}
\texttt{Banker-TINF} applied to delayed adversarial MABs ensures 
$\mathfrak{R}_T = \O\left(\sqrt{KT} + \sqrt{KD\log D}\right)$.
\end{corollary}

\begin{remark}
While the delay-related term $\O(\sqrt{KD\log D})$ is slightly worse than the SOTA result $\O(\sqrt{D\log K})$ \citep{zimmert2020optimal}, we see that \texttt{Banker-OMD} conveniently translates  algorithm design into  finding a  regularizer with appropriate regularity, which is a task decoupled from handling feedback delays. As we will present in the following sections, this feature enables us to conveniently adapt existing OMD-based algorithms from a large variety of non-delayed bandits settings to their delayed counterparts.
\end{remark}

\section{Application I: Delayed Scale-Free MAB}\label{sec:delayed scale-free MAB}

In prior sections of this paper, we assumed by default that all losses are $[0,1]$-bounded. However, in practice, it is often hard to quantize the feedback to a limited range in a deterministic way. For example, one advertisement may cause unexpectedly great reactions in the market, either positive or negative. Therefore, the $[0,1]$-bounded-loss assumption may fail to capture such subtle but important scenarios --- moreover, as one will see shortly, algorithm design in such a setting is hard, even in absence of feedback delays.

In a scale-free adversarial MAB problem (without feedback delays), all losses $l_{t,a}$'s are within some bounded interval $[-L, L]$ and $L$ is  \textit{unknown} to the agent.\footnote{The assumption on the boundedness of the interval is without loss of generality because $L$ is kept as a secret: viewing from hindsight, such an $L<\infty$ always exists.}
Due to the $\Omega(\sqrt{KT})$ lower bound for standard adversarial MABs \citep{auer2002nonstochastic}, a trivial lower bound $\Omega(\sqrt{KT}L)$ exists. On the other hand, the best-known upper bound by \citet{putta2021scale} only reduce to $\Otil(K\sqrt{T}L)$ in the worst case --- a $\sqrt K$ gap with the lower bound exists even without feedback delays (see \Cref{remark:comparision with putta} for more discussion).

A scale-free MAB algorithm is more robust to extreme feedback, but we want more --- e.g., being robust to feedback delays. This results in the novel setting which we call delayed scale-free adversarial MAB, where the feedback of each round $t$ will only be delivered after $d_t$ rounds; meanwhile the losses still fall into an unknown bounded range $[-L,L]$ instead of $[0,1]$.
This is again the reality: in advertisement recommendations, the effect of deploying an advertisement also cannot be immediately observed.

\subsection{High-Level Design Ideas}
\label{sec-sftinf-high-level-ideas}

In this section, we outline the design idea of our algorithm. For this section, we shall assume the losses to be non-negative --- this assumption enables us to use the $\nicefrac 12$-Tsallis entropy regularizer, following the idea of \Cref{corollary-adversarial}.\footnote{For the general loss case, the log-barrier regularizer can be used to derive an even better regret guarantee (which adapts to the losses), albeit with a more complicated analysis; see \Cref{sec:sflbinf}.}

As scale-free losses already produce adequate difficulties \citep{putta2021scale}, we first present our solution for non-delayed settings before presenting the delay-robust algorithms. 
Assume that we face a non-delayed instance where $d_t\equiv 0$ and we additionally have a \textit{forecast} oracle about $\lVert l_t\rVert_\infty$ at the beginning of each round $t$ (i.e., $ l_{t,A_t}\in [0,\lVert l_t\rVert_\infty]$ always holds). With the non-delay assumption and the forecast, we can follow the spirit of non-delayed adversarial MAB algorithms \citep{zimmert2019optimal} for an algorithm with $\O(L\sqrt{KT})$ regret, as follows:

\textbf{Idealized Setting \#1 (no delay, with $\lVert l_t\rVert_\infty$ forecast).}
With the $\nicefrac 12$-Tsallis entropy regularizer, the total immediate costs become $\sqrt K\sum \sigma_t^{-1} l_{t,A_t^2}$. Analogue to the $[0,1]$-bounded case \citep{zimmert2019optimal}, we want it to be of order $\sigma_t$ after applying the summation lemma (\Cref{lem:summation lemma}). Thus, \begin{equation*}
    \sigma_t = \left (1 + \sum_{s=1}^t \Vert l_s \Vert_\infty^2\right )^{\nicefrac 1 2}
\end{equation*}
ensures the total immediate costs to be bounded by
\begin{align*}
&\quad \sqrt K\sum_{t=1}^T \Vert l_t \Vert_\infty^2\left (1 + \sum_{s=1}^t \Vert l_s \Vert_\infty^2\right )^{-\nicefrac 12}\\
&=\O \left (\sqrt K \sqrt {1 + \sum_{s=1}^T \Vert l_s \Vert_\infty^2} \right ) = \O(\sqrt{KT}L).
\end{align*}
The same $\sigma_t$ also ensures the total investment cost to be $\sigma_T \sqrt K=\O(\sqrt{KT}L)$. Hence, $\mathfrak R_T=\O(\sqrt{KT} L)$.

The above reasoning shows that we can easily adapt a classical $[0,1]$-valued MAB algorithm to scale-free settings with the help of accurate forecasts. In fact, with such forecasts available, we can also generalize \texttt{Banker-TINF} for delayed adversarial MABs (\Cref{thm-banker-tinf}) to scale-free tasks.

\textbf{Idealized Setting \#2 (has delay, with $\lVert l_t\rVert_\infty$ forecast).}
Inspired by \Cref{thm-banker-omd-B}, we set \begin{equation*}
    \sigma_t = (\mathfrak d_t + 1)^{-1} \cdot \left ( 1 + \sum_{s=1}^t (\mathfrak d_s + 1)\Vert  l_s \Vert_\infty^2 \right )^{\nicefrac 1 2},
\end{equation*}
where we weigh the delays in $\mathfrak D_t$ by the loss magnitudes. 
The total immediate costs then become proportional to 
\begin{align*}
& \quad \sum_{t=1}^T \E [\sigma_t^{-1}l_{t,A_t}^2] \\
& \le \sum_{t=1}^T \lVert l_t\rVert_\infty^2(\mathfrak d_t+1)\cdot \left (1+\sum_{s=1}^t (\mathfrak d_s+1) \lVert l_s\rVert_\infty^2\right )^{-\nicefrac 12}\\
& =\O\left (\sqrt{\sum_{t=1}^T (\mathfrak d_t+1)\lVert l_t\rVert_\infty^2}\right )=\O(\sqrt{D+T}L),
\end{align*}
where we again used \Cref{lem:summation lemma} together with the fact that $\sum_t(\mathfrak d_t+1)\le D+T$.
Following the proof of \Cref{thm-banker-omd-B}, we can bound the total investment as
$B_T=\Otil \left (\sqrt{\sum_{t=1}^T (\mathfrak d_t+1)\lVert l_t\rVert_\infty^2} \right )=\Otil(\sqrt{D+T}L)$.
Combining these two parts then gives $\mathfrak R_T=\Otil(\sqrt{K(D+T)}L)$.


\textbf{Actual Setting.} In the actual situation without the $\Vert l_t \Vert_\infty$ oracle, we introduce a doubling trick to maintain an upper bound $\hat L_t$ of the maximum observed feedback $l_{s,A_s}$. We then replace all unknown $\lVert l_s\rVert_\infty$'s in $\sigma_t$ with $\hat L_t$, pretending that it is a prediction returned by an ideal oracle, i.e.,
\begin{equation}\label{eq:sigma in SFTINF}
\sigma_t=(\mathfrak d_t+1)^{-1}\cdot \left (1+\sum_{s=1}^t (\mathfrak d_s+1)\hat L_s^2\right )^{\nicefrac 12}.
\end{equation}

While maintaining the estimations $\hat L_t$, if we receive feedback $l_{s,A_s}$ exceeding the current $\hat L_t$ (i.e., we underestimated the real losses), we update the upper bound $\hat L$ to $2 l_{s,A_s}$ and \textit{skip} that round by setting the loss estimator $\tilde l_s$ to $\mathbf 0$ (see \Cref{line-sftinf-skip}).
As there is at most $\O(\sqrt D)$ feedback on the way, each doubling only causes $\O(\sqrt D)$ rounds to be skipped. Their losses can be bounded by the new value of $\hat L_t$, which means the total regret caused by skipping is at most $\sum_{\hat L_t}\O(\sqrt D L_t)=\O(\sqrt DL)$ --- informally, by doing so, we get a ``weak oracle'' that upper bounds each $\lvert l_{t,A_t}\rvert$ within a constant multiplicative error, which is sufficient for $\Otil(\sqrt T+\sqrt D)$-style regret as we show in \Cref{sec:sketch of SFTINF}.

\subsection{\texttt{Banker-SFTINF} for Non-Negative Losses}
We formalize the ideas presented in the last section into \Cref{banker-sftinf}, with a regret guarantee stated in \Cref{theorem-sftinf-regret}.

\begin{algorithm}[!t]
\caption{\texttt{Banker-SFTINF} for Delayed Scale-Free Adversarial MAB with Non-Negative Losses}
\label{banker-sftinf}
\begin{algorithmic}[1]
\STATE Initialize $\hat L_1 = 1$. Deploy \texttt{Banker-OMD} (with modifications) with $x_0 = \nicefrac {\mathbf 1} K$ and the $\nicefrac 12$-Tsallis entropy regularizer $\Psi(x) = -2\sum_{i=1}^K \sqrt{x_i}$ as follows.
\FOR{$t=1,2,\ldots,T$}
\STATE Set $a_t \leftarrow \mathrm{missing}$ and $\mathfrak d_t\gets \sum_{s=1}^{t-1}\mathbbm 1_{\{a_s=\mathrm{missing}\}}$.
\STATE Set $D_t\gets \sum_{s=1}^t (\mathfrak d_s + 1)\hat L_s^2$. \COMMENT{implement \Cref{eq:sigma in SFTINF}}
\STATE Calculate $\sigma_t=\left ((\mathfrak d_t+1)\sqrt\frac{\ln(3 + D_t/\hat L_t^2)}{{3+D_t}}\right )^{-1}$.
\STATE Pick new action $x_t$ as Lines 5 -- 10 of \Cref{banker-omd}.
\STATE Play $A_t \sim x_t$ and initialize $\hat L_{t+1}\gets \hat L_t$.
\FOR{upon receiving each new feedback $(s, l_{s,A_s})$}
    \STATE $\hat L_{t+1} \gets \max\{\hat L_{t+1}, 2l_{s,A_s}\}$. \COMMENT{doubling trick} \alglinelabel{line-sftinf-maintain-hat-L}
    \IF{$l_{s,A_s} > \hat L_s$}
        \STATE $\tilde l_s \gets \mathbf 0$; $a_s \gets \mathrm{skipped}$. \COMMENT{skip $s$}  \alglinelabel{line-sftinf-skip}
    \ELSE
        \STATE $\tilde{l}_s \leftarrow \frac{l_{s,A_s}}{x_{s,A_s}}\mathbf{1}_{A_s}$; $a_s \leftarrow\mathrm{arrived}$. \COMMENT{keep $s$}
    \ENDIF
    \STATE Calculate $z_s \leftarrow \nabla\bar{\Psi}^*(\nabla\Psi(x_s) - \frac{1}{\sigma_s}\tilde{l}_s)$.
\ENDFOR
\ENDFOR
\end{algorithmic}
\end{algorithm}


\begin{theorem}\label{theorem-sftinf-regret}
When the losses are non-negative, the \texttt{Banker-SFTINF} algorithm in \Cref{banker-sftinf} ensures
\begin{align*}
\mathfrak R_T
& = \O\left (\sqrt{K(D+T)\log (D+T)}\cdot L\right ) \\
& =\Otil\left (\sqrt{K(D+T)}L\right ).
\end{align*}
\end{theorem}

\begin{remark}
As noticed by \citet{cesa2016delay}, any delayed adversarial MAB algorithm must suffer $\Omega(\sqrt{KT}+\sqrt{D\log K})$ regret when losses are $[0,1]$-bounded. Thus, in the scale-free setting, any reasonable algorithm must suffer $\Omega((\sqrt{KT}+\sqrt{D\log K})\cdot L)$ total regret (even assuming a known $L$). Hence, regarding $K$ as a small constant independent of $T$, \Cref{banker-sftinf} is only $ \O(\sqrt{\log (D+T)})=\Otil(1)$ times worse than the theoretical lower bound.
\end{remark}

\subsection{Analysis Sketch of \texttt{Banker-SFTINF}}\label{sec:sketch of SFTINF}
We sketch the analysis of \texttt{Banker-SFTINF}. A formal version is presented in \Cref{sec:proof banker-sftinf}.
Denote by $\mathcal E_t$ the event that 
$a_t=\mathrm{skipped}$'' eventually.
Let $\hat l_t \triangleq \frac{l_{t,A_t}}{x_{t, A_t}} \mathbf 1_{A_t}$ be the unnullified estimator. We decompose the regret as follows:
\begin{align}
    \mathfrak R_T 
    & = \sum_{t=1}^T \E[\mathbbm 1_{\neg \mathcal E_t}\langle \hat l_t, x_t - \mathbf 1_{i^*} \rangle] + \sum_{t=1}^T \E[\mathbbm 1_{\mathcal E_t}\langle \hat l_t, x_t - \mathbf 1_{i^*} \rangle] \nonumber\\
    & \le \underbrace{\sum_{i=1}^T \E[\langle \tilde l_t, x_t - \mathbf 1_{i^*} \rangle]}_{\text{Banker-OMD controlled regret}} + \underbrace{\sum_{t=1}^T \E[\mathbbm 1_{\mathcal E_t}l_{t,A_t}]}_{\lVert l_t\rVert_\infty\text{ estimation error}} \label{eq-sftinf-decomposition}
\end{align}
because $\tilde l_t = \mathbbm 1_{\neg \mathcal E_t} \hat l_t$ by definition. Since $\tilde l_t$'s are the losses fed into \texttt{Banker-OMD}, the first term decomposes into the total investment cost $\E[B_T]\cdot D_\Psi(y,x_0)$ and the immediate costs $\sum_{t=1}^T \E[\mathbbm 1_{\neg \mathcal E_t} \sigma_t D_\Psi(x_t,\tilde z_t)]$, thanks to \Cref{thm-banker-omd}.

First, consider the total investment $B_T$. Let the last new investment be made at $T_0$ (\Cref{lemma-banker-omd-B}). Suppose that there are $m$ feedback corresponding to rounds $T_1,T_2,\ldots,T_m$ absent at $T_0$ (i.e., $T_i+d_{T_i}\ge T_0$). Then $B_T$ is bounded by
\begin{align*}
B_T &\le \sum_{i=0}^m \frac{1}{\mathfrak d_{T_i}+1}\sqrt{3+D_{T_i}}\\
& =\Otil\left (\sqrt{(D+T)L^2}\right ),
\end{align*}
because $m=\O(\sqrt D)$ and $\sum_{t=1}^T (\mathfrak d_t+1)=\O(D+T)$.

For immediate costs, we can prove $\E[\sigma_t D_\Psi(x_t,\tilde z_t)]\le \sqrt K \sigma_t^{-1} \lVert l_t\rVert_\infty^2$ analogue to the $[0,1]$-bounded cases (formally stated in \Cref{lemma-apdx-sftinf-single-immediate-cost}).
Moreover, we have
\begin{align*}
\sum_{t=1}^T \sigma_t^{-1}\lVert l_t\rVert_\infty^2 & \le \sqrt{\sum_{t=1}^T (\mathfrak d_t+1)\hat L_t^2}\\
& \le \sqrt{(D+T)}L
\end{align*}
due to the summation lemma \Cref{lem:summation lemma} and $\lVert l_t\rVert_\infty\le \hat L_t$.

Besides, the other term in \Cref{eq-sftinf-decomposition}
can never exceed $O(\sqrt D L)$ --- intuitively, each ``skipping'' will only cause no more than $\sqrt D$ feedback on the way to be skipped.
Therefore, \Cref{eq-sftinf-decomposition} gives $\mathfrak R_T=\Otil(\sqrt{K(D+T)}L)$, as claimed.

\subsection{\texttt{Banker-SFLBINF} for General Losses}\label{sec:sflbinf}
As mentioned, the $\nicefrac 12$-Tsallis entropy regularizer only works for non-negative losses. Fortunately, for the general-loss case, we can use the log-barrier regularizer $\Psi(x)=-\sum_{i=1}^K \ln x_i$ to derive the \texttt{Banker-SFLBINF} algorithm (\Cref{banker-sflbinf} in \Cref{sec:design of banker-sflbinf}).
It not only allows us to derive the same regret guarantee as \Cref{theorem-sftinf-regret} when losses can be negative (up to logarithmic factors), but also gives a small-loss style \citep{neu2015first} adaptive regret bound.
The bound is stated in \Cref{theorem-sflbinf-regret} and proved in \Cref{sec:proof banker-sflbinf}.


\begin{theorem}\label{theorem-sflbinf-regret}
\texttt{Banker-SFLBINF} in \Cref{banker-sflbinf} ensures 
\begin{equation*}
    \mathfrak R_T=\Otil \left (\sqrt{K\E[\tilde{\mathfrak L}_T^2]}+\sqrt{KD}L\right ),
\end{equation*}
where $\tilde {\mathfrak L}^2_T \triangleq \sum_{t=1}^T (\mathfrak d_t + 1) l_{t,A_t}^2$ and
$\mathfrak d_t$ is the number of feedback absent at the beginning of round $t$.
In particular, 
\begin{align*}
    \mathfrak R_T
    & = \O\left(\sqrt{K(D+T)}\log T L + \sqrt D L\log L \right)
    \\
    & =\Otil\left (\sqrt{K(D+T)}L\right ).
\end{align*}
\end{theorem}

\begin{remark}\label{remark:comparision with putta}
When running on non-delayed instances, \texttt{Banker-SFLBINF} enjoys an $\Otil(\sqrt{KT}L + KL)$ regret guarantee. Meanwhile, the SOTA bounds are $\tilde{\mathcal O}(\sqrt{KL_2}+L_\infty \sqrt{KT})$ or $\Otil(\sqrt{KL_2}+L_\infty \sqrt{KL_1})$ \citep{putta2021scale}, which both reduce to $\Otil(K\sqrt TL)$ in the worst case. Hence, our algorithm dominates theirs by a factor of $\sqrt K$ in the worst case, closing the gap with the $\Omega(\sqrt{KT}L)$ lower bound \citep{auer2002nonstochastic}.
More importantly, our bound is \textit{uniformly} better than theirs because $\tilde{\mathfrak L}_T^2=\sum_{t=1}^T l_{t,A_t}^2\le \sum_{t=1}^T \lVert l_t\rVert_2^2=L_2$.
Besides, a technical comparison with \citep{putta2021scale} is presented in \Cref{sec:comparison with putta}.
\end{remark}


\section{Application II: Delayed Linear Bandits}
\label{sec-application-linear}
In a delayed adversarial linear bandit, there is a convex action set $\mathcal{C}\subseteq \mathbb{R}^n$ where $n$ is called the ambient dimension. Correspondingly, there is a loss set $\bar{\mathcal{C}}=\{l\in \mathbb{R}^n \mid \lvert \langle l, x\rangle \rvert \le 1, \forall x \in \mathcal{C}\}$. At each round $t$, the agent picks an action $A_t \in \mathcal{C}$. At the same time, an adversary picks a loss vector $l_t \in \bar{\mathcal{C}}$ and a delay $d_t$. Then the agent suffers a loss of $\hat{l}_t = \langle l_t, A_t \rangle$ which is revealed to the agent at the end of the $t+d_t$-th round. The objective of the agent is still to minimize the pseudo-regret, now defined by
\begin{equation*}
\mathfrak{R}_T \triangleq  \sup_{y \in \mathcal{C}} \E\left [\sum_{t=1}^T \langle l_t, A_t - y \rangle\right ].
\end{equation*}

OMD, as in many other problems, has been widely adopted in non-delayed linear bandit problems \citep{abernethy2008competing, bubeck2012towards, bubeck2015entropic}. In particular, \citet{abernethy2008competing} proposed an OMD-based algorithm \texttt{BOLO} for adversarial linear bandits. It uses an $\O(n)$-self-concordant barrier (see \Cref{def:self-concordant} for an exact definition) together with a sampling scheme supported on the Dikin ellipsoid (which we will introduce in \Cref{banker-bolo}). This algorithm achieves $\O(n^{\nicefrac 32}\sqrt{T\log T})$ regret.

Now, we illustrate how to easily convert the \texttt{BOLO} algorithm into a delay-robust version via our \texttt{Banker-OMD} framework. We first pick an $\O(n)$-self-concordant $\Psi$ and set $x_0=\nabla\Psi^*(\mathbf{0})$; according to \citet{abernethy2008competing}, this ensures $(\Psi, x_0)$ to be $(\O(n\log T), \O(n^2))$-regular under the sampling scheme of \texttt{BOLO}.
After that, we pick the action scales according to \Cref{thm-banker-omd-B} -- resulting in the \texttt{Banker-BOLO} algorithm (presented as \Cref{banker-bolo} in \Cref{sec:design of banker-bolo}).
Our \texttt{Banker-BOLO} algorithm then works in delayed adversarial linear bandits and ensures the following:

\begin{theorem}
\label{thm-banker-bolo}
When $\Psi$ is an $\O(n)$-self-concordant barrier over $\mathcal{C}$, \texttt{Banker-BOLO} in \Cref{banker-bolo} ensures 
\begin{equation*}
    \mathfrak{R}_T=\Otil \big (n^{\nicefrac 32}\sqrt T+n^2\sqrt D\big ),
\end{equation*}
which is only $\Otil(n^2\sqrt D)$ more than the non-delayed regret $\Otil(n^{\nicefrac 32}\sqrt T)$ achieved by \texttt{BOLO} \citep{abernethy2008competing}.
\end{theorem}

To our knowledge, \texttt{Banker-BOLO} is the first algorithm achieving $\Otil(\text{poly}(n)(\sqrt{T} + \sqrt{D}))$ regret for adversarial linear bandits with arbitrary delays.%
\footnote{Meanwhile, the lower bound is still under-explored. The only lower bound for delayed adversarial linear bandits is $\Omega(n\sqrt T+\sqrt{ndT})$ under the uniform and known delay assumption \citep{ito2020delay}. As our setting is more general, $(\sqrt T+\sqrt D)$ factors are unavoidable, though the optimal dependency on $n$ remains unknown.}
The previous feedback delay model on adversarial linear bandits only allows \textit{constant} delay lengths \citep{ito2020delay}; moreover, this constant length needs to be revealed to the agent \textit{beforehand}.  Therefore, our setting is far more general than the previous one.
And -- more importantly -- such achievements are attained by simply plugging a non-delayed adversarial linear bandit algorithm (\texttt{BOLO}) into our \texttt{Banker-OMD} framework.

\section{Conclusion and Future Directions}
In this paper, we propose the \texttt{Banker-OMD} framework for bandit learning problems with feedback delays. It almost decouples feedback delay handling and the task-specific OMD algorithm design.
We illustrate the power of our framework by studying two different problems --- we achieve the first near-optimal performance in both settings. Our result also extends to non-delayed scale-free adversarial MABs.

However, the most important contribution of our work is to  illustrate the power of \texttt{Banker-OMD} in decoupling feedback delays from the algorithm design (i.e., deciding the regularizer $\Psi$ and action scale $\sigma_t$'s). Using our framework, algorithms for non-delayed problems like \texttt{Tsallis-INF} \citep{zimmert2019optimal} or \texttt{BOLO} \citep{abernethy2008competing} can be easily tuned to handle delays. Thus, we expect our framework to be used in many other delayed bandit learning problems.

Moreover, \texttt{Banker-OMD} allows the learning scales to be non-monotonic (see \Cref{banker-omd}), which is more flexible than vanilla OMD even in non-delayed settings.
Our framework is also capable of handling arm-dependent delays (see \Cref{apdx-arm-dependent-delays}).
We leave further investigations to the future.

Besides regret, the space complexity of \texttt{Banker-OMD} can also be improved: The current version requires memorizing the actions for all rounds with absent feedback (which can be as large as $\O(\sqrt DK)$), while the algorithms specially designed for delayed MABs only need $\O(K)$ space \citep{zimmert2020optimal}. See \Cref{sec:efficiency} for more discussions.

\section*{Acknowledgements}

This work is supported by the Technology and Innovation Major Project of the Ministry of Science and Technology of China under Grant 2020AAA0108400 and 2020AAA0108403 and the Tsinghua Precision Medicine Foundation 10001020109.
We want to thank the anonymous reviewers for their insightful comments.

\newpage
\bibliographystyle{ACM-Reference-Format}
\bibliography{references}

\onecolumn
\newpage
\appendix
\renewcommand{\appendixpagename}{\centering \LARGE Supplementary Materials}
\appendixpage

\startcontents[section]
\printcontents[section]{l}{1}{\setcounter{tocdepth}{2}}

\section{More Discussions}
\begin{table*}[htb]
\begin{minipage}{\textwidth}
\caption{Overview of Our Algorithms and Some Related Works}
\label{tab:table-regret}
\begin{savenotes}
\renewcommand{\arraystretch}{1.6}
\resizebox{\textwidth}{!}{%
\begin{tabular}{|c|c|c|}\hline
Algorithm & Setting & Regret Upper bound \\\hline
\citet{zimmert2020optimal} & \multirow{2}{*}{\shortstack{Delayed Adversarial MAB\\ $\Omega(\sqrt{KT}+\sqrt{D\log K})$ \citep{cesa2016delay}}} & $\Otil(\sqrt{KT}+\sqrt D)$ \\\cline{1-1} \cline{3-3}
\texttt{Banker-TINF} (\textbf{\Cref{thm-banker-tinf}}) & & $\Otil(\sqrt{K(D+T)})$ \\\hline
\texttt{Banker-SFTINF} (\textbf{\Cref{banker-sftinf}}) & \multirow{2}{*}{\shortstack{Scale-free Delayed Adversarial MAB\\ $\Omega((\sqrt{KT}+\sqrt{D\log K})L)$ \citep{cesa2016delay}}} & $\Otil(\sqrt{K(D+T)}L)$ for non-negative losses \\\cline{1-1} \cline{3-3}
\multirow{2}{*}{\texttt{Banker-SFLBINF} (\textbf{\Cref{banker-sflbinf})}} &  & $\Otil\left (\sqrt{K\E[\tilde{\mathfrak L}_T^2]}+\sqrt {KD}L\right )=\Otil(\sqrt{K(D+T)}L)$ \\\cline{2-3}
& \multirow{2}{*}{\shortstack{Scale-free Adversarial MAB\\ $\Omega(\sqrt{KT}L)$ \citep{auer2002nonstochastic}}} & $\Otil\left (\sqrt{K\E[\sum_{t=1}^T l_{t,A_t}^2]}\right )=\Otil(\sqrt{KT}L)$ \\\cline{1-1} \cline{3-3}
\citet{putta2021scale} &  & $\Otil\left (\sqrt{K\sum_{t=1}^T \lVert l_t\rVert_2^2}+L\sqrt{KT}\right )=\Otil(K\sqrt TL)$ \\\hline
\texttt{Banker-BOLO} (\textbf{\Cref{banker-bolo}}) & \multirow{2}{*}{\shortstack{Delayed Adversarial Linear Bandit\\ $\Omega(n\sqrt T+\sqrt{nD})$ \citep{ito2020delay}\footnote{Precisely, \citet{ito2020delay} gave an $\Omega(n\sqrt T+\sqrt{ndT})$ bound for the uniform delay case where the delay $d$ is known before-hand.}}} & $\Otil(n^{3/2}\sqrt T+n^2 \sqrt D)$ \\\cline{1-1} \cline{3-3}
\citet{ito2020delay} &  & $\Otil(n\sqrt T+\sqrt n \sqrt{dT})$ with known uniform delays $d$ \\\hline
\end{tabular}}
\end{savenotes}
\end{minipage}
\end{table*}

\subsection{More Related Works}
\label{sec:related work}
\textbf{Online Mirror Descent.} Online mirror descent is a concept that originated in classical optimization and first developed by \citet{nemirovski1979efficient} and \citet{nemirovskij1983problem}. It was introduced to the context of online learning no later than \citet{warmuth1997continuous, gordon1999regret, shalev2007online, shalev2007primal}. Since then, OMD has been used in various online learning tasks with adversarial reward, such as MABs \citep{audibert2014regret, abernethy2015fighting, wei2018more, zimmert2019optimal}, semi-bandits \citep{audibert2014regret, zimmert2019beating}, linear bandits \citep{abernethy2008competing, bubeck2012towards, audibert2014regret} and Markov Decision Processes (MDP) \citep{jin2020learning, jin2020simultaneously}. For more discussions, one may refer to the references therein.

\textbf{Delayed MAB.} \citet{joulani2013online} first studied adapting existing stochastic MAB algorithm to delayed settings with i.i.d. delays. \citet{cesa2016delay} studied adversarial MABs with uniform delays $d$ and derived a regret lower bound of $\Omega(\max\{\sqrt{KT}, \sqrt{dT\log K}\})$. 
 Recent works \citep{bistritz2019online, thune2019nonstochastic, zimmert2020optimal} show that the total overhead due to feedback delays can be controlled in $\O(\sqrt{D\log K})$ (where $D$ is the total delay, which is $dT$ if the delays are uniform) by introducing a moving negative entropy component in the OMD regularizer. Remarkably, all these works assumed $[0,1]$-bounded losses, i.e., $l_{t,a}\in [0,1]$, $\forall t\in [T],a\in [K]$. Thus, our delayed scale-free MAB setting extends existing results to allow general losses. 
 %

 
\textbf{Delayed Linear Bandits.} \citet{NEURIPS2019_56cb94cb} first studied stochastic linear bandits with i.i.d. delays. \citet{vernade2020linear} then studied stochastic linear bandits with random but only partially observable delays. \citet{ito2020delay} studied adversarial linear bandits with uniform delay $d$, proposing an $\Otil(\sqrt{n(n+d)T})$-regret algorithm and showing a near-matching $\Omega(\sqrt{n(n+d)T})$ lower bound. 

\textbf{Learning with Feedback Delays.}
Aside from aforementioned delayed MABs and delayed linear bandits, feedback delays are also considered in stochastic Markov Decision Processes (MDPs), see, e.g., \citep{lancewicki2020learning,howson2021delayed,jin2022near,dai2022follow}.
General online optimizations with feedback delays and full information \citep{mesterharm2005line,agarwal2011distributed} or bandit feedback \citep{dudik2011efficient,desautels2014parallelizing} are also extensively studied in the literature.
Finally, feedback delays in wireless network optimizations are recently handled by \citet{huang2021robust}, via the application of a preliminary version of our \texttt{Banker-OMD} framework.

\textbf{Scale-free Learning.} Scale-free settings were first studied in full-information case (e.g., \citep{de2014follow,orabona2018scale}). The most recent work in this line \citep{orseau2021isotuning} studies non-delayed full-information scale-free online learning.
For MABs, \citet{hadiji2020adaptation} proposed an $\mathcal O(L_\infty \sqrt{KT\log K})$ algorithm for stochastic ones.
\citet{putta2021scale} then considered the more challenging adversarial case, yielding two adaptive bounds, $\Otil(\sqrt{KL_2}+L_\infty \sqrt{KT})$ and $\Otil(\sqrt{KL_2}+L_\infty \sqrt{KL_1})$, which both become $\Otil(K\sqrt{T} L_\infty)$ in the worst case. Here, $L_p$ denotes the $p$-norm of the flattened loss matrix, i.e., $L_p=(\sum_{t=1}^T \sum_{i=1}^K l_{t,i}^p)^{1/p}$. Importantly, no previous work considered delayed feedback in scale-free MABs, and we give the first attempt to solve this problem.

\subsection{Why Not Use Existing Works to Solve the Delayed Scale-free Adversarial MAB Problem?}
\label{sec:zimmert paper}
In this section, we explain why we are not satisfied with simply using existing works (e.g., \citep{zimmert2020optimal}) to solve our delayed scale-free adversarial MAB setting (via techniques like doubling trick).
The main reason is that handling delays and handling scale-free losses are actually coherent --- a largely delayed large loss can make the loss estimation inaccurate for a long phase. As a result, to obtain an analysis for the modified algorithm, it still needs to resolve a few challenging technical difficulties.

For example, consider the representative work by \citet{zimmert2020optimal}. This work introduces the notation $\hat L_{t,i_t}^{\text{miss}}$ to denote the sum of estimated losses of the outstanding prior steps. Originally, their analysis upper-bounds the expectation of $\hat L_{t,i_t}^{\text{miss}}$ by $\mathfrak d_t$. As a result, the dependence of $D$ in final regret bound comes from summing up $\hat L_{t,i_t}^{\text{miss}}$ and then taking expectation. While this argument works well when $L=1$, such calculation is no longer valid if they get divided by $\hat L_t$'s, as a heavily delayed large loss will affect up to $\Theta(D)$ subsequent $\hat L_{t,i_t}^{\text{miss}}$'s. Moreover, when there may be negative losses, $\hat L_{t,i_t}^{\text{miss}}$ may contain negative components, which means the step (f) in the proof of Lemma 3 (which uses second-order remainder upper bounds of Bergman divergences) no longer holds.

Therefore, as resolving such issues are challenging and does not provide generality and flexibility on other tasks (like delayed adversarial linear bandits), we instead propose a new framework of handling delayed bandit learning problems --- as we illustrated in the main text, our framework allows various nice properties and solves many different problems (most of which are open problems in the literature). That's why we do not directly adopt existing delayed MAB papers.

\subsection{Arm-Dependent Delays and \texttt{Banker-OMD}}
\label{apdx-arm-dependent-delays}
Throughout the paper, we discussed about feedback delays that only depends on time indices (i.e., they are irrelevant to the actions actually taken).
This model, namely arm-independent delay model, is the most common one in the literature \citep{thune2019nonstochastic,bistritz2019online,zimmert2020optimal,gyorgy2021adapting}.
However, in realistic applications, such an assumption is often too restrictive. For example, in medical experiments, different medicine may take different time to show its effect.

To resolve this issue, \citet{van2022nonstochastic} recently proposed a more general setting called \textit{arm-dependent delays}, where there is a delay matrix $\delta \in \mathbb N^{T\times K}$ (also chosen by an oblivious adversary) to give out delay lengths for each (round, arm) \textit{combination}.
Formally, the actual delay at time $t$ is determined by $d_t \gets \delta_{t,A_t}$. 

\citet{van2022nonstochastic} pointed out that existing adversarial MAB works on \textit{arm-indepdendent} delays, e.g., \citep{thune2019nonstochastic,zimmert2020optimal} are \textit{not capable} of dealing with the new setting. They also proposed a new algorithm guaranteeing
\begin{equation}\label{eq:arm-dependent delays}
    \mathfrak R_T \le \sqrt{KT} + \sqrt{\ln K \cdot \E\left [\sum_{t=1}^T\langle x_t, \rho_t \rangle\right ]} + \rho^\ast,
\end{equation}
where $\rho_t$ is a $K$-dimensional vector such that $\rho_{t,a}$ denotes the number of feedback of action $a$ that is still on the way at the beginning of round $t$, i.e., $\rho_{t,a}\triangleq \sum_{s=1}^{t-1} \mathbbm 1_{A_s=a,s+\delta_{s,a}\ge t}$. Moreover, $\rho^\ast \triangleq \max_{t\in [T], a\in [K]} \rho_{t,a}$ denotes the maximum number of missing observations of some action.

Although our \texttt{Banker-OMD} framework is presented with the assumption of arm-independent delays, we claim here that our results also apply to arm-dependently delayed settings as long as we redefine the symbol $D$ as the sum of all delay lengths \textit{actually experienced}, namely $D\triangleq \sum_{t=1}^T \delta_{t,A_t}$. In other words, $D$ becomes a random variable rather than a obliviously chosen constant --- therefore, we can attain a regret bound in terms of $\E[D]$, recovering \Cref{eq:arm-dependent delays}.

Technically speaking, the reason for the difference between our result and the previous ones \citep{thune2019nonstochastic,zimmert2020optimal} is that, rather than directly bound the regret (which involves expectation), our core result \Cref{thm-banker-omd} upper bounds the \textit{sample path} loss difference $\sum_{t=1}^T \langle \tilde{l}_t, x_t - y\rangle$.
When we bound the total investment namely $B_T$, we also derive a sample path upper bound for $B_T$: As we only make use of bounds like $\sum_{i\le n} 1/i = \Theta(\log n)$ and monotonicity of functions like $x/\ln x$, whether $D$ is a random variable does not affect this bound of $B_T$.
Moreover, as for each immediate cost term $\sigma_t D_\Psi(x_t, \tilde z_t)$, we either directly derive a sample path upper bound (e.g., for log-barrier $\Psi$, \Cref{lemma-immedicate-cost-general-log-barrier}), or derive a bound for the conditional expectation $\E[ \sigma_t D_\Psi(x_t, \tilde z_t) \mid \mathcal F_{t-1}]$, we can always obtain a $\O(\sigma_t^{-1})$ bound given appropriate regularity conditions of $\Psi$.
We then sum up these bounds and then apply summation lemmas to get bounds in $D$. Again, until now, whether $D$ is a random variable does not matter.
At this point, we have derived an upper bound for $\sum_{t=1}^T \langle \tilde{l}_t, x_t - y\rangle$ in $D$. Luckily, this bound is a concave function in $D$ (as it only involves square roots and logarithms). Taking expectations on both sides, we can then get a regret bound in $\E[D]$, just like \Cref{eq:arm-dependent delays}.

\subsection{Technical Comparison with \citep{putta2021scale}}
\label{sec:comparison with putta}
Compared to our work, \citet{putta2021scale} did not explicitly estimate the actual loss range $L$ (whereas we did by introducing $\hat L_t$'s).
Hence, instead of our tuning $\sigma_t \propto (1+\sum_{s=1}^{t-1}l_{s,A_s}^2 + \hat L_t^2)^{\nicefrac 12}\approx (1+\sum_{s=1}^{{\color{blue}{t}}}l_{s,A_s}^2)^{\nicefrac 12}$, they can only set $\sigma_t \propto (1+\sum_{s=1}^{{\color{blue}{t-1}}}l_{s,A_s}^2)^{\nicefrac 12}$ -- thus, they can only use the following summation lemma different from \Cref{lem:summation lemma}: 
\begin{equation*}
\sum_{i=1}^n \frac {x_i}{\sqrt{1 + \sum_{j=1}^{i-1}x_j}} \le \O\left (\sqrt{1 + \sum_{i=1}^n x_i} + \max_{1\le i\le n} x_i\right ).
\end{equation*}

The maximum of $x_i$ causes an extra term that is proportional to the maximum reciprocal of probabilities to pull each arm in their regret bound.
Consequently, their algorithm must include \textit{explicit uniform exploration}, i.e., playing $(1-\gamma_t)x_t+\frac{\gamma_t}{K}$ instead of $x_t$ for the $t$-th step. Even for non-delayed settings, $\gamma_t =\Omega(t^{-\nicefrac 1 2})$ is needed for $\tilde \O(\sqrt{T}L)$ style total regret. Therefore, if one directly adapts their technique to delayed settings, the regret bound would involve a term proportional to
\begin{equation*}
\max_{1\le t\le T} \sum_{s\le t: a_s=\text{``missing'' at }t} x_{s,A_s}^{-1},
\end{equation*}

which could be as large as $\Theta(\sqrt{DT})$ in the worst case, making their framework incapable to feedback delays.

In a nutshell, there are essentially two improvements in our delayed scale-free MAB results: the first one is due to the \textit{novel learning rates} (which improves the regret in non-delayed settings), and the second one is due to \textit{our proposed Banker-OMD framework} (which easily generalizes the non-delayed regret bounds to delay-robust ones).


\subsection{Computational and Space Complexity of \texttt{Banker-OMD}}
\label{sec:efficiency}
\textbf{Computational Complexity.}
Notice that the loop in Lines 8-9 of \Cref{banker-omd} is to greedily spend up previous savings to meet the desired new action scale $\sigma_t$. Thus, by maintaining a linked list of all ``non-exhausted'' savings, the loop only needs to scan amortized $\O(1)$ previous time indices to decide each new $x_t$. Hence the algorithm displayed in \Cref{banker-omd} can be implemented with an amortized time complexity of $\O(K)$ per action -- which matches vanilla OMD.

\textbf{Space Complexity.}
While the computational complexity of \Cref{banker-omd} is the same as vanilla OMD, the space complexity is much larger as we need to memorize all previous actions.
Fortunately, we can also adopt a slightly different decision rule of $\sigma_{t,s}$ and $b_t$ (which is not included in the current version for ease of presentation), which further reduces the space complexity to $\O(\sqrt DK)$ while maintaining a per-step $\O(K)$ amortized time complexity:
\begin{enumerate}
\item Calculate the sum of all available savings $v=\sum_{s<t,a_s\ =\ \text{arrived}}v_s$.
\item If $v<\sigma_t$, fill up the remaining proportion by $b_t=\sigma_t-v$. Otherwise, let $b_t=0$. This step ensures that \Cref{lemma-banker-omd-B} (the upper bound on the total investment $B_T$) still holds.
\item Finally, for each $s<t$ where $a_s=\text{arrived}$, we set $\sigma_{t,s}=\min(v,\sigma_t)\frac{v_s}{v}$. In other words, we set $\sigma_{t,s}\propto v_s$ such that $\sum_{s<t,a_s\ =\ \text{arrived}}\sigma_{t,s}=\min(v,\sigma_t)$, which ensures \Cref{eq:conditions of b and m}.
\end{enumerate}

Under the new policy, $x_t$ in Line 10 of \Cref{banker-omd} becomes
\begin{equation*}
x_t=\nabla \bar{\Psi}^\ast\left (\frac{\min(v,\sigma_t)}{\sigma_t}\sum_{s<t,a_s\ =\ \text{arrived}}\frac{v_s}{v}\nabla \Psi(z_s)+\frac{b_t}{\sigma_t}\nabla \Psi(x_0)\right ).
\end{equation*}

Therefore, by tracking $\sum_{s<t,a_s\ =\ \text{arrived}}\frac{v_s}{v}\nabla \Psi(z_s)$, we can obtain $x_t$ by applying the mirror map $\bar \Psi^*(\cdot)$ to the current $\sum_{s<t,a_s\ =\ \text{arrived}}\frac{v_s}{v}\nabla \Psi(z_s)$. Because we always have $v_s\propto \sigma_s$, each update does not involve calculating the summation again and thus can be done in $\mathcal O(K)$ time -- again the same as vanilla OMD.
Moreover, in this version, we only need to memorize the actions whose feedback is absent, giving $\mathcal O(\sqrt DK)$ space consumption. 

\section{Technical Details for Online Mirror Descent and \texttt{Banker-OMD}}\label{sec:appendix analysis of banker}
Throughout the paper, we use
\begin{equation*}
    \mathcal{F}_t = \sigma\left(A_1, \ldots, A_t, l_{1,A_1}\mathbbm{1}[d_1 + 1 \le t], \ldots, l_{t,A_t}\mathbbm{1}[d_t + t \le t]\right)
\end{equation*}
to denote the filtration of $\sigma$-algebra  when studying random quantities indexed by time. Below, we present a summation lemma that we heavily use when tuning the learning scales.

\begin{lemma}[Summation Lemma {\citep[Lemma 3.5]{auer2002adaptive}}]\label{lem:summation lemma}
The following holds for any non-negative $x_1,x_2,\ldots,x_T$'s:
\begin{equation*}
    \sum_{t=1}^T \frac{x_t}{\sqrt{1+\sum_{s=1}^tx_s}}\le 2\sqrt{1+\sum_{t=1}^T x_t}.
\end{equation*}
\end{lemma}

\subsection{Legendre Functions and Their Properties}
\begin{definition}[Legendre Functions]\label{def:legendre}
For a strictly convex $f$ defined on a convex domain $A\subseteq \mathbb R^K$, we say $f$ is \textit{Legendre} if (i)  $\mathop{int}(A)$ is non-empty, (ii) $f$ is differentiable on $\mathop{int}(A)$, and (iii) $\lim_{n\rightarrow\infty}\Vert\nabla f(x_n)\Vert_2 \rightarrow \infty$ for any sequence $(x_n)_{n=1}^\infty$ in $\mathop{int}(A)$ converging to some $x\in \partial \mathop{int}(A)$.
\end{definition}

The proof of the following lemmas can be found in any textbook for convex analysis, e.g., \citep{rockafellar2015convex}.

\begin{lemma}
\label{lemma-legendre}
Let $\mathcal{C}\subseteq \mathbb{R}^n$ be a convex set, $f:\mathcal{C} \rightarrow \mathbb{R}$ be a Legendre function. Then,
\begin{enumerate}
    \item $\nabla f$ is a bijection between $\mathop{int}(\mathcal{C})$ and $\mathop{int}(\mathop{dom}(f^*))$ with the inverse $(\nabla f)^{-1} = \nabla f^*$;
    \item $D_f(y, x) = D_{f^*}(\nabla f(x), \nabla f(y))$ for all $x, y \in \mathop{int}(\mathcal{C})$;b
    \item the convex conjugate $f^*$ is Legendre.
\end{enumerate}
\end{lemma}

We now give a formal proof to the single-step OMD regret bound \Cref{eq-omd-basic-mab}, stated in the following lemma.

\begin{lemma}
\label{lemma-omd-basic}
For any $\sigma > 0$, $x,y \in \triangle^{[K]}$, $l \in \mathbb{R}_+^K$ and Legendre function $\Psi: \mathbb{R}_+^{K}\ \rightarrow \mathbb{R}$, we have
\begin{equation}
    \left\langle l, x - y \right\rangle \le \sigma D_\Psi(y, x) - \sigma D_\Psi(y, z) + \sigma D_\Psi(x, \tilde{z}) \label{eq-lemma-omd-basic}
\end{equation}
where
\begin{equation}
    z = \mathop{\arg\min}_{x' \in \triangle^{[K]}} \left\langle l, x' \right\rangle + \sigma D_\Psi(x', x),\quad
    \tilde{z} = \mathop{\arg\min}_{x' \in \mathbb{R}_+^K} \left\langle l, x' \right\rangle + \sigma D_\Psi(x', x), \label{eq-lemma-omd-basic-z-def}
\end{equation}
or equivalently,
\begin{equation}
    z = \nabla\bar{\Psi}^*(\nabla\Psi(x) - \frac{1}{\sigma}l),\quad
    \tilde{z} = \nabla\Psi^*(\nabla\Psi(x) - \frac{1}{\sigma}l). \label{eq-lemma-omd-basic-z-def-alt}
\end{equation}
\end{lemma}
\begin{proof}
$\Psi$ is Legendre hence $\nabla\Psi$ explodes on $\partial \mathbb{R}_+^K$, which guarantees that the minimizer $x'$ in the definition of $\tilde{z}$ in \Cref{eq-lemma-omd-basic-z-def} will lie in $\mathop{int}(\mathbb{R}_+^K)$ and $\frac{\partial}{\partial x'}[\langle l, x'\rangle + \sigma D_\Psi(x', x)] = 0$. The bijection property in \Cref{lemma-legendre} then asserts this $\mathop{\arg\min}$ definition is equivalent to the definition in \Cref{eq-lemma-omd-basic-z-def-alt} using mirror maps $\nabla \Psi$ and $\nabla \Psi^*$. Since $\bar{\Psi}$ is a Legendre function on $\triangle^{[K]}$, similar argument suggests that the definitions for $z$ in \Cref{eq-lemma-omd-basic-z-def} and \Cref{eq-lemma-omd-basic-z-def-alt} are equivalent.

The definition of $\tilde{z}$ in \Cref{eq-lemma-omd-basic-z-def-alt} implies $l = \sigma (\nabla\Psi(x) - \nabla\Psi(\tilde{z}))$. The first order optimality condition of $z$ in \Cref{eq-lemma-omd-basic-z-def} implies that $\langle \frac{1}{\sigma}l + \nabla\Psi(z) - \nabla\Psi(x), y - z \rangle \ge 0$ for any $y \in \triangle^{[K]}$. Hence we have
\begin{align*}
    \langle l, x - y \rangle & = \langle l, x-z \rangle + \langle l, z - y \rangle \\
    & \le \sigma \langle \nabla\Psi(x) - \nabla\Psi(\tilde{z}), x - z \rangle + \sigma \langle \nabla\Psi(z) - \nabla\Psi(x), y - z \rangle \\
    & \stackrel{(a)}{=} \sigma(D_\Psi(z, x) + D_\Psi(x, \tilde{z}) - D_\Psi(z, \tilde{z})) - \sigma(D_\Psi(y, z) + D_\Psi(z, x) - D_\Psi(y, x)) \\
    & = \sigma D_\Psi(y, x) - \sigma D_\Psi(y, z) + \sigma D_\Psi(x, \tilde{z}) - \sigma D_\Psi(z, \tilde{z}) \\
    & \le \sigma D_\Psi(y, x) - \sigma D_\Psi(y, z) + \sigma D_\Psi(x, \tilde{z})\end{align*}
where $(a)$ uses the following ``three-point identity'' of Bregman divergences:
\begin{equation*}
    D_\Psi(a, b) + D_\Psi(b, c) - D_\Psi(a, c) = \langle \nabla\Psi(c) - \nabla\Psi(b), a-b \rangle.
\end{equation*}
\end{proof}

\subsection{Detailed Proofs in \Cref{sec-banker-omd}}
First of all, we study when the immediate cost terms $\E[\sigma_t D_\Psi(x_t, \tilde{z}_t)]$ in \Cref{vanilla-omd,banker-omd} can be uniformly bounded:
\begin{lemma}
\label{lemma-mab-immediate}
    If $\Psi(x) = \sum_{i=1}^K f(x_i)$ where $f''(x) = x^{-\alpha}$, $\alpha \ge 1$, then in \Cref{vanilla-omd,banker-omd}, for any $t\ge 1$ we have
    \begin{equation*}
        \E[ \sigma_t D_\Psi(x_t, \tilde{z}_t) \mid \mathcal{F}_{t-1}] \le \frac{K}{2\sigma_t}.
    \end{equation*}
\end{lemma}
\begin{proof}
In fact, for any choice of $\Psi$ and $t\ge 1$ we have
\begin{align}
    \sigma_t D_\Psi(x_t, \tilde{z}_t) & = \sigma_t D_{\Psi^*}(\nabla\Psi(\tilde{z}_t), \nabla\Psi(x_t)) \nonumber \\
    & = \sigma_t D_{\Psi^*}(\nabla\Psi(x_t) - \frac{\tilde{l}_t}{\sigma_t}, \nabla\Psi(x_t)) \nonumber \\
    & = \sigma_t\left(Psi^*(\nabla\Psi(x_t) - \frac{\tilde{l}_t}{\sigma_t}) - \Psi^*(\nabla\Psi(x_t)) - \langle x_t,  - \frac{\tilde{l}_t}{\sigma_t}\rangle \right)\nonumber \\
    & = \frac{\Vert \tilde{l}_t \Vert_{\nabla^2\Psi^*(\theta_t)}^2}{2\sigma_t}, \label{eq-immediate-cost-matrix-norm}
\end{align}
where in the last step we write the Bregman divergence into a second order Lagrange remainder, $\theta_t$ is some element inside the line segment connecting $\nabla\Psi(x_t) - \frac{1}{\sigma_t}\tilde{l}_t$ and $\nabla\Psi(x_t)$.

Note that under this particular $\Psi(x) = \sum_{i=1}^K f(x_i)$, we have
\begin{itemize}
    \item $\nabla^2\Psi^*(\cdot)$ is diagonal,
    \item The diagonal elements of $\nabla^2\Psi^*(\cdot)$ are non-decreasing on the line segment $[\nabla\Psi(x_t) - \frac{1}{\sigma_t}\tilde{l}_t, \nabla\Psi(x_t)]$,
\end{itemize}
and we can further upper bound \Cref{eq-immediate-cost-matrix-norm} by $\frac{1}{2\sigma_t} \Vert \tilde{l}_t \Vert_{\nabla^2\Psi^*(\nabla\Psi(x_t))}^2 = \frac{1}{2\sigma_t} \Vert \tilde{l}_t \Vert_{\nabla^2\Psi(x_t)^{-1}}^2$. Therefore,
\begin{align*}
    \E[\sigma_t D_\Psi(x_t, \tilde{z}_t) \mid \mathcal{F}_{t-1}] & \le \E[\frac{1}{2\sigma_t} \Vert \tilde{l}_t \Vert_{\nabla^2\Psi(x_t)^{-1}}^2 \mid \mathcal{F}_{t-1}] \\
    & = \frac{1}{2\sigma_t} \sum_{i=1}^K x_{t,i}^{\alpha}\E[\tilde{l}_{t,i}^2 \mid \mathcal{F}_{t-1}] \\
    & = \frac{1}{2\sigma_t} \sum_{i=1}^K x_{t,i}^{\alpha}\E\left[\frac{l_{t,A_t}^2\mathbbm{1}[A_t = i]}{x_{t,i}^2} \Bigm\vert \mathcal{F}_{t-1}\right] \\
    & = \frac{1}{2\sigma_t} \sum_{i=1}^K x_{t,i}^{\alpha - 1} \\
    & \le \frac{K}{2\sigma_t}.
\end{align*}
\end{proof}

\subsubsection{Proof of \Cref{lemma-banker-convex-comb}}
We give a formal proof for the ``convex combination in the dual space'' lemma of \texttt{Banker-OMD} (\Cref{lemma-banker-convex-comb}) here.

\begin{proof}[Proof of \Cref{lemma-banker-convex-comb}]
Let $\tilde{x} = \nabla\Psi^*(\sum_{i=1}^h \frac{\sigma_i}{\sigma}\nabla\Psi(z_i))$, we have
\begin{align*}
    \sigma D_\Psi(y, x) & \stackrel{(a)}{\le} D_\Psi(y, \tilde{x}) \\
     & \stackrel{(b)}{=} \sigma D_{\Psi^*}(\nabla\Psi(\tilde{x}), \nabla\Psi(y)) \\
     & = \sigma D_{\Psi^*}(\sum_{i=1}^h \frac{\sigma_i}{\sigma}\nabla\Psi(z_i), \nabla\Psi(y)) \\
     & \stackrel{(c)}{\le} \sigma \cdot \sum_{i=1}^h \frac{\sigma_i}{\sigma} D_{\Psi^*}(\nabla\Psi(z_i), \nabla\Psi(y)) \\
     & \stackrel{(d)}{=} \sum_{i=1}^h \sigma_i D_\Psi(y, z_i)
\end{align*}
where $(a)$ is due to the Pythagorean theorem for Bregman divergences ($D_\Psi(y, \tilde{x}) = D_\Psi(y, x) + D_\Psi(x, \tilde{x}) \ge D_\Psi(y, x)$), $(b)$ is due to the duality property of Bregman divergences, $(c)$ is due to the convexity of the first argument of Bregman divergences, and $(d)$ uses again the duality property.
\end{proof}

\subsubsection{Proof of \Cref{thm-banker-omd}}
It is then easy to see that the general regret bound for \texttt{Banker-OMD} (\Cref{thm-banker-omd}) is just a summation over the single-step regret bounds in \Cref{lemma-banker-convex-comb}. A formal proof is presented here.

\begin{proof}[Proof of \Cref{thm-banker-omd}]
When $T=0$ this bound trivially holds. Suppose inequality \Cref{eq-thm-banker-omd} holds for all $T \le M$, then the difference of \Cref{eq-thm-banker-omd}'s RHS between $T = M+1$ and $T = M$ is at least
\begin{align*}
    & \quad (B_{T+1} - B_T)D_\Psi(y, x_0) + \sigma_{T+1}D_\Psi(x_{T+1}, \tilde{z}_{T+1}) - \sum_{i=1}^T \triangle v_i D_\Psi(y, z_i) - \sigma_{T+1}D_\Psi(y, z_{T+1}) \\
    & = (b_{T+1}D_\Psi(y, x_0) + \sum_{i=1}^T \sigma_{T+1,i} D_\Psi(y, z_i)) - \sigma_{T+1}D_\Psi(y, z_{T+1}) + \sigma_{T+1}D_\Psi(x_{T+1}, \tilde{z}_{T+1}) \\
    & \stackrel{(a)}{\ge} \sigma_{T+1}D_\Psi(y, x_{T+1}) - \sigma_{T+1}D_\Psi(y, z_{T+1}) + \sigma_{T+1}D_\Psi(x_{T+1}, \tilde{z}_{T+1}) \\
    & \stackrel{(b)}{\ge} \langle \tilde{l}_{T+1}, x_{T+1} - y \rangle,
\end{align*}
which is the difference of \Cref{eq-thm-banker-omd}'s LHS between $T = M+1$ and $T = M$. Here we use $\triangle v_i$ to denote the change of the variable $v_i$ in \Cref{banker-omd} from the end of time $T$ to the end of time $T+1$, $(a)$ is due to \Cref{lemma-banker-convex-comb}, $(b)$ is due to \Cref{lemma-omd-basic}. Thus by induction, we are done.
\end{proof}

\subsubsection{Proof of \Cref{lemma-banker-omd-B}}
\begin{proof}[Proof of \Cref{lemma-banker-omd-B}]
\Cref{lemma-banker-omd-B} is a special case of the following fact: at the end of any time $T$, $B_T = \sum_{i=1}^T v_i$. This fact can be easily verified by an induction on $T$. If \texttt{GreedyPick} is used, when we encounter a new round $T$ and observe $b_T > 0$, it is guaranteed that at that moment, any non-zero $v_i$ is corresponding to an action whose feedback is still on the way (including the action $A_{T}$ we are going to play).
\end{proof}

\subsubsection{Proof of \Cref{thm-banker-omd-B}}
The general method for $\O((\sqrt{T} + \sqrt{D\log D})f(T))$ regret bounds (\Cref{thm-banker-omd-B}) can be proved following the same proof sketch we presented in the text to deal with $\sigma_t = (\frac{1}{\sqrt{t}} + \frac{\mathfrak{d}_t}{\sqrt{\mathfrak{D}_t}})^{-1}$. To be specific,
\begin{proof}[Proof of \Cref{thm-banker-omd-B}]
When \texttt{Banker-OMD} uses \texttt{GreedyPick} and action scales $\sigma_t = \left (\frac{1}{\sqrt{t}} + \mathfrak{d}_t\sqrt{\frac{\ln( \mathfrak{D}_t+1)}{\mathfrak{D}_t}}\right )^{-1}$, we have
\begin{align*}
    \sum_{t=1}^T \sigma_t^{-1} & = \sum_{t=1}^T \left (\frac{1}{\sqrt{t}} + \mathfrak{d}_t\sqrt{\frac{\ln( \mathfrak{D}_t+1)}{\mathfrak{D}_t}}\right ) \\
    & \le \sum_{t=1}^T \left (\frac{1}{\sqrt{t}} + \sqrt{\ln(D + 1)}\frac{\mathfrak{d}_t}{\sqrt{\mathfrak{D}_t}}\right ) \\
    & = \O(\sqrt{T} + \sqrt{D\log D})
\end{align*}
where in the last step, we use the fact that $\mathfrak D_t$ is a cumulative sum of $\mathfrak d_1\ldots, \mathfrak d_t$ and apply the summation lemma \Cref{lem:summation lemma}.

In order to bound $B_T$, consider the last moment $T_0$ at which $B_t$ gets increased, and suppose at the beginning of $T_0$ there are $m=\mathfrak{d}_{T_0}$ feedback on the way, each corresponding to actions at time $t_1 < \cdots < t_m$ respectively. Then it is easy to see that $\mathfrak{d}_{t_i} \ge i-1$ for any $1\le i\le m$, for at the beginning of time $t_i$, actions invoked at time $t_1,\ldots,t_{i-1}$ are still on the way. Furthermore, we have $D \ge \mathfrak{d}_{T_0} + \sum_{i=1}^m \mathfrak{d}_{t_i} = \binom{m+1}{2}$ hence $m = \O(\sqrt{D})$.

$\sigma_t = \left (\frac{1}{\sqrt{t}} + \mathfrak{d}_t\sqrt{\frac{\ln( \mathfrak{D}_t+1)}{\mathfrak{D}_t}}\right )^{-1}$ guarantees that $\sigma_t \le \sqrt{t}$ and $\sigma_t \le \frac{1}{\mathfrak{d}_t} \sqrt{\frac{\mathfrak{D}_t}{\ln \mathfrak{D}_t + 1}} = \O\left (\sqrt{\frac{D}{\log D}}\right )\frac{1}{\mathfrak{d}_t}$, where the last step is due to the monotonicity of $\sqrt{\frac{x}{\ln x + 1}}$ when $x$ is sufficiently large. According to \Cref{lemma-banker-omd-B}, we have 
\begin{align*}
    B_T & = \sigma_{T_0} + \sum_{i=1}^m \sigma_{t_i} \\
        & \le \O\left (\sqrt{\frac{D}{\log D}}\right )\frac{1}{\mathfrak{d}_{T_0}} + \sqrt{t_1} + \sum_{i=2}^m \O\left (\sqrt{\frac{D}{\log D}}\right )\frac{1}{\mathfrak{d}_{t_i}} \\
        & \le \sqrt{T} + \O\left (\sqrt{\frac{D}{\log D}}\right )\left (\sum_{i=2}^m \frac{1}{i-1} + \frac{1}{m}\right ) \\
        & = \O(\sqrt{T} + \sqrt{D \log D}).
\end{align*}
\end{proof}

\section{Technical Details for \texttt{Banker-SFTINF}}\label{sec:proof banker-sftinf}

By plugging \Cref{thm-banker-omd} into \Cref{eq-sftinf-decomposition}, we get the following regret decomposition for \Cref{banker-sftinf}:

\begin{lemma}
    \Cref{banker-sftinf} guarantees
    \begin{equation}
        \mathfrak R_T \le \underbrace{ B_T \sup_{y \in \triangle^{[K]}}D_\Psi(y, x_0))}_{\text{Total investment}} + \underbrace{\sum_{t=1}^T \E[\sigma_t D_\Psi(x_t, \tilde z_t)] }_{\text{Immediate costs}} + \underbrace{\sum_{t=1}^T \E[\mathbbm 1_{\mathcal E_t} l_{t,A_t}]}_{\text{Skipping error}}
        \label{eq-sbtinf-regret-decomposition-apdx}
    \end{equation}
    where $\mathcal E_t$ denotes the event that the feedback of round $t$'s action is marked to ``skipped'', $\sigma_t$, $B_t$, $x_t$, $\tilde z_t$ are variables determined by \Cref{banker-sftinf}.
\end{lemma}

Below, we will bound the three terms in \Cref{eq-sbtinf-regret-decomposition-apdx}  within $\O(L\sqrt{K(D+T)\log(D+T)})$ one by one. 

\subsection{Bound for Total Investment Term}

For the total investment term $B_T \sup_{y \in \triangle^{[K]}}D_\Psi(y, x_0))$, recall that we choose $\Psi$ to be the $\nicefrac 12$-Tsallis entropy function, which guarantees that $\sup_{y \in \triangle^{[K]}}D_\Psi(y, x_0) \le 2\sqrt K$. Thus it suffices to bound $B_T$. Applying \Cref{lemma-banker-omd-B}, suppose the last round on which we make a new investment is $T_0$ and that time, there are $m$ previous feedback corresponding to rounds $T_1,\ldots, T_m$ still on the way, then we have
\begin{align*}
    B_T & = \sum_{i=0}^m \sigma_{T_i} \\
    & = \sum_{i=0}^m \frac 1 {\mathfrak d_{T_i} + 1} \sqrt \frac {3 + D_{T_i}} {\ln(3 + D_{T_i} / \hat L_{T_i}^2)} \\
    & \le \left(\sum_{i=0}^m \frac 1 {\mathfrak d_{T_i} + 1}\right) \cdot \max_{1\le t\le T} \sqrt \frac {3 + D_t} {\ln(3 + D_t / \hat L_t^2)} \\
    & \le \left(\sum_{i=0}^m \frac 1 {i + 1}\right) \cdot \max_{1\le t\le T} \hat L_t \sqrt \frac {3 + D_t / \hat L_t^2} {\ln(3 + D_t / \hat L_t^2)} \\
    & \stackrel{(a)}{\le} \O(\log m) \cdot \O\left( L \sqrt \frac {D+T} {\ln (D+T)}\right) \\
    & \le \O(\log D) \cdot \O\left( L \sqrt \frac {D+T} {\ln (D+T)}\right) \\
    & \le \O\left( L \sqrt {(D+T)\log D}\right)
\end{align*}
where in step $(a)$ we simply control the $\hat L_t$ factor out of the square root by $L$, and utilize the fact that $D_t / \hat L_t \le t + \sum_{s=1}^t \mathfrak d_t$ for all $t$ and the function $(3+x) / \ln (3 + x)$ is increasing. Thus we have show the total investment is within $\O\left( L \sqrt {K(D+T)\log D}\right)$. 

\subsection{Bound for Immediate Costs}

To deal with the immediate costs, we first show each immediate cost term $E[\sigma_t D_\Psi(x_t, \tilde z_t)]$ can be bounded very similarly compared to the $[0,1]$-bounded loss case.

\begin{lemma}
\label{lemma-apdx-sftinf-single-immediate-cost}
    Let $\Psi$ be the $\nicefrac 12$-Tsallis entropy function, $x_t \in \triangle^{[K]}$, $A_t$ be an independent sample from $[K]$ according to $x_t$, $\hat l_t \in \mathbb R_+^{K}$, $\tilde l_t = \frac {\hat l_{t,A_t}}{x_{t,A_t}}\mathbf 1_{A_t}$. Let $\tilde z_t \triangleq \nabla \Psi^*(\nabla\Psi(x_t) - \frac{\tilde{l}_t}{\sigma_t})$ then we have
    \begin{equation*}
        \sigma_t D_\Psi(x_t, \tilde z_t) \le \frac{\hat l_{t,A_t}^2 x_{t,A_t}^{-\nicefrac 1 2}}{\sigma_t}
    \end{equation*}
    and
    \begin{equation*}
        \E[\sigma_t D_\Psi(x_t, \tilde z_t)] \le \sqrt K \frac {\lVert \hat l_t\rVert_\infty^2}{\sigma_t}.
    \end{equation*}
\end{lemma}
\begin{proof}
    If suffices to follow the calculation we used in 
the proof of \Cref{lemma-mab-immediate}:
\begin{align}
    \sigma_t D_\Psi(x_t, \tilde{z}_t) & \stackrel{(a)}{=} \sigma_t D_{\Psi^*}(\nabla\Psi(\tilde{z}_t), \nabla\Psi(x_t)) \nonumber \\
    & = \sigma_t D_{\Psi^*}(\nabla\Psi(x_t) - \frac{\tilde{l}_t}{\sigma_t}, \nabla\Psi(x_t)) \nonumber \\
    & = \Psi^*(\nabla\Psi(x_t) - \frac{\tilde{l}_t}{\sigma_t}) - \Psi^*(\nabla\Psi(x_t)) - \langle x_t,  - \frac{\tilde{l}_t}{\sigma_t}\rangle \nonumber \\
    & \stackrel{(b)}{=} \frac{\Vert \tilde{l}_t \Vert_{\nabla^2\Psi^*(\theta_t)}^2}{2\sigma_t} \nonumber
\end{align}
where $(a)$ is due to the duality of Bregman divergences, and $(b)$ regards the Bregman divergence as a second order Lagrange remainder, $\theta_t$ is some element inside the line segment connecting $\nabla\Psi(x_t) - \frac{1}{\sigma_t}\tilde{l}_t$ and $\nabla\Psi(x_t)$. 

Here our $\Psi$ is the $\nicefrac 12$-Tsallis entropy function, then $\nabla^2\Psi^*(\cdot)$ is diagonal. We can verify that $\Psi_{ii}^{*\prime\prime}(\theta) = \Psi_{ii}^{\prime\prime}(\Psi_i^{\prime -1}(\theta))^{-1} = 2((-\theta)^{-2})^{\nicefrac 3 2} = -2\theta^{-3}$ is increasing for all $i$ and $\theta < 0$. Also notice that all components of $\tilde l$ are non-negative, hence we have $\nabla^2\Psi^*(\theta_t) \preceq \nabla^2\Psi^*(\nabla \Psi(x_t)) = \nabla^2\Psi(x_t)^{-1}$, and therefore
\begin{equation*}
    \sigma_t D_\Psi(x_t, \tilde{z}_t) \le \frac{\Vert \tilde{l}_t \Vert_{\nabla^2\Psi(x_t)^{-1}}^2}{2\sigma_t} = \frac{\tilde l_{t,A_t}^2 x_{t,A_t}^{\nicefrac 3 2}}{\sigma_t} = \frac{\hat l_{t,A_t}^2 x_{t,A_t}^{-\nicefrac 1 2}}{\sigma_t} \le \frac{\lVert \hat l_t \rVert_\infty^2 x_{t,A_t}^{-\nicefrac 1 2}}{\sigma_t}.
\end{equation*}
Taking expectation on both sides of the inequality, we can further get
\begin{equation*}
    \E[\sigma_t D_\Psi(x_t, \tilde{z}_t)] \le \frac{\lVert \hat l_t \rVert_\infty^2}{\sigma_t} \sum_{i=1}^K x_{t,i}^{\nicefrac 1 2} \le \sqrt K \frac{\lVert \hat l_t \rVert_\infty^2}{\sigma_t}.
\end{equation*}
\end{proof}

For the immediate costs term, the \textit{after-skipping} estimator $\tilde l_t$ used in \texttt{Banker-SFTINF} can be regarded as an importance sampling estimate for $\hat l_t \triangleq l_t \cdot \mathbbm 1_{\{l_{t,A_t} \le \hat L_t \}}$, which is a vector in $\mathbb R_+^K$, thus \Cref{lemma-apdx-sftinf-single-immediate-cost} applies, giving that
\begin{equation*}
    \E[\sigma_t D_\Psi(x_t, \tilde z_t) \mid \mathcal F_{t-1}] \le \frac{\E[\hat l_{t,A_t}^2 x_{t,A_t}^{-\nicefrac 1 2 }]}{\sigma_t} \le \frac {\hat L_t^2 \E[x_{t,A_t}^{-\nicefrac 1 2}]}{\sigma_t} \le \sqrt K \frac{\hat L_t^2}{\sigma_t}.
\end{equation*}
Therefore, it remains to bound $\sum_{t=1}^T \E[\frac {\hat L_t^2} {\sigma_t}]$.
We can write
\begin{align*}
    \sum_{t=1}^T \frac{\hat L_t^2}{\sigma_t} 
    & = \sum_{t=1}^T \frac{(\mathfrak d_t + 1)\hat L_t^2}{\sqrt{3 + D_t}} \sqrt{\ln(3 + D_t / \hat L_t^2)} \\
    & \stackrel{(a)}{\le}\O(\sqrt {\log T}) \sum_{t=1}^T \frac{(\mathfrak d_t + 1)\hat L_t^2}{\sqrt{3 + D_t}} \\
    & \stackrel{(b)}{\le} \O(\sqrt {\log T}) \cdot \O(\sqrt{1 + D_T}) \\
    & \le  \O(L\sqrt{(D + T)\log T}),
\end{align*}
where in step $(a)$, we make use of the fact that $D_t/\hat L_t^2 \le t + \sum_{s=1}^t \mathfrak d_s$ for all $1 \le t \le T$ to bound $\sqrt{\ln(3 + D_t / \hat L_t^2)}$ universally for all $t$ by $\sqrt{\ln(3 + D + T)} = \O(\log T)$; in step $(b)$, recall that $D_t$ is just the sum of all $(\mathfrak d_s + 1)\hat L_s^2$ for $s=1,2,\ldots,t$, thus we can apply the summation lemma \Cref{lem:summation lemma}. Thus we have verified that the immediate costs are within $\O(L\sqrt{K(D+T)\log T})$.

\subsection{Bound for Skipping Error}
\label{sec-sftinf-skipping-error-apdx}

It remains to bound the skipping error term $\sum \E[\mathbbm 1_{\mathcal E_t}l_{t,A_t}]$. For simplicity, denote $l_{t,A_t}$ by $\hat l_t$, also let $\mathcal T \triangleq \{t\in [T]: \neg \mathcal E_t\}$ denote the set of all skipped time indices. We begin with the following important observation:
\begin{lemma}
\label{lemma-skipped-steps}
For any indices $s,t\in \mathcal T$, if $\hat l_t < 2\hat l_s$, we have $s+d_s \ge t$, i.e., the feedback of action $A_s$ arrives later than the decision $A_t$ being made.
\end{lemma}
\begin{proof}
Otherwise $\hat l_s$ would be used to update $\hat L_{s+d_s + 1}$, thus $\hat L_{s+d_s + 1} \ge 2\hat l_s$, but $t \ge s + d_s + 1$ so we should have $\hat L_t \ge \hat L_{s+d_s + 1} \ge 2\hat l_s$, a contradiction. 
\end{proof}

Suppose $\lvert \mathcal T \rvert = m$ and denote by $t_1, t_2, \cdots, t_m$ the elements of $\mathcal T$, which are sorted in the \textit{arrival order} of their feedback. Consider partitioning $\mathcal T$ into non-empty subsets $\mathcal T_1 \triangleq \{t_1,t_2,\ldots, t_{m_1}\}, \mathcal T_2 \triangleq\{t_{m_1 + 1},t_{m_1 + 2},\ldots,t_{m_2}\}, \cdots, \mathcal T_m \triangleq\{t_{m_{k-1}+1},\ldots,t_{m_k}\}$ (specially, we denote $m_0=0$, $m_k=m$), with the following two properties hold:
\begin{itemize}
    \item (intra-subset loss upper bound) for each $0\le i < k$, we have $\hat l_{t_j} < 2\hat l_{t_{m_i + 1}}$ for all $m_i< j \le m_{i+1}$;
    \item (inter-subset loss lower-bound) for each $0\le i < k-1$, we have $\hat l_{m_{i+1}+1} \ge 2\hat l_{m_i + 1}$.
\end{itemize}
Such partition does exist and can be found by a greedy scan over $\mathcal T$. The intra-subset and inter-subset loss requirments allow us to conclude that the number of subsets $k$ is  $\O(\log L)$. Besides, for each subset $\mathcal T_i$, Lemma~\ref{lemma-skipped-steps} states that all time indices in this subset are no later than the arrival of the feedback of time $t_{m_i + 1}$, namely $t_{m_i + 1} + d_{t_{m_i + 1}}$. Therefore, $\lvert\mathcal T_i\rvert \le d_{t_{m_i + 1}} + 1 \le \mathcal O(\sqrt{D})$.

After making these observations, one can bound the total losses incurred by rounds in each $\mathcal T_i$ by
\begin{align*}
    \sum_{t \in \mathcal T_i} \hat l_t & \stackrel{(a)}{\le} 2\lvert\mathcal T_i\rvert \hat l_{t_{m_i + 1}}
    \\
    & \stackrel{(b)}{\le} 4\cdot 2^{-(k-i)}\lvert\mathcal T_i\rvert \hat l_{t_{m_{k-1} + 1}} \\
    & \stackrel{(c)}{\le} 4\cdot 2^{-(k-i)} \cdot \mathcal O(\sqrt{D}) \cdot L \\
    & \le \mathcal O(2^{-(k-i)}\sqrt{D} L)
\end{align*}
where in step $(a)$ we make use of the intra-subset loss relationship in $\mathcal T_i$, in step $(b)$ we apply intra-subset loss relationship for $k-i-2$ times to control $\hat l_{t_{m_i + 1}}$ by $\hat l_{t_{m_{k-1} + 1}}$, in step $(c)$ we apply $\lvert\mathcal T_i\rvert \le \mathcal O(\sqrt{D})$.

Thus the total loss incurred by $\mathcal T$ , i.e., the skipping error, is bounded by $\mathcal O(\sqrt{D} L)$.

\section{Technical Details for \texttt{Banker-SFLBINF}}\label{sec:proof banker-sflbinf}
In this section, we introduce our algorithm \texttt{Banker-SFLBINF} and also analyze its performance.
\subsection{Algorithm Design}\label{sec:design of banker-sflbinf}

\begin{algorithm}[htb]
\caption{\texttt{Banker-SFLBINF} for Delayed Scale-Free Adversarial MAB}
\label{banker-sflbinf}
\begin{algorithmic}[1]
\STATE Initialize $\hat L_1 = 1$, {\color{blue}$\mathfrak{D}_0 \gets 1$}
\STATE Run \texttt{Banker-OMD} with $x_0 = \nicefrac {\mathbf 1} K$, {\color{blue}the log-barrier regularizer $\Psi(x) = -\sum_{i=1}^K \ln x_i$} and \texttt{GreedyPick} scheduler, with some additional processing at each round as below.
\FOR{$t=1,2,\ldots$}
    \STATE $a_t \leftarrow \mathrm{missing}$, $\mathfrak d_t\gets \sum_{s=1}^{t-1}\mathbbm 1_{\{a_s=\mathrm{missing}\}}$.
    \STATE $\color{blue}\mathfrak D_t \gets \mathfrak D_{t-1} + \mathfrak d_t$. \COMMENT{Also maintain ordinary experienced total delays}
    \STATE $D_t\gets \color{blue} \sum_{s\in \{1,\ldots,t\}:a_s = \mathrm{missing}} (\mathfrak d_s + 1)\hat L_s^2 + \sum_{s\in \{1,\ldots,t\}:a_s = \mathrm{arrived}} (\mathfrak d_s + 1) l_{s,A_s}^2$. \alglinelabel{line-alg2-D}
    \STATE $\sigma_t=\left ((\mathfrak d_t+1)\sqrt\frac{\ln(3 + D_t/\hat L_t^2)}{{3+D_t}} {\color{blue}\sqrt{K\ln T}} \right)^{-1}$. \alglinelabel{line-alg2-sigma-before-max}
    {\color{blue}\IF{$\mathfrak d_t \le \sqrt{\nicefrac{\mathfrak D_t} K}$}
        \STATE $\sigma_t \gets \max\{\sigma_t, 2\hat L_t\}$ \alglinelabel{line-alg2-sigma}
    \ENDIF}
    \STATE Pick new action $x_t$ as indicated by \Cref{banker-omd}. Play $A_t \sim x_t$.
    \STATE Initialize $\hat L_{t+1}\gets \hat L_t$.
    \FOR{upon receiving each new feedback $(s, l_{s,A_s})$}
        \STATE $\hat L_{t+1} \gets \max\{\hat L_{t+1}, 2\lvert l_{s,A_s}\rvert\}$. \alglinelabel{line-sflbinf-maintain-hat-L}
        \IF{$\lvert l_{s,A_s}\rvert > \hat L_s$ {\color{blue} or $l_{s,A_s} < -\frac 1 2 \sigma_s$ \alglinelabel{line-alg2-skip}}}
            \STATE $\tilde l_s \gets \mathbf 0$; $a_s \gets \mathrm{skipped}$. \COMMENT{an additional skipping criterion}
        \ELSE
            \STATE $\tilde{l}_s \leftarrow \frac{l_{s,A_s}}{x_{s,A_s}}\mathbf{1}_{A_s}$; $a_s \leftarrow\mathrm{arrived}$
        \ENDIF
        \STATE $z_s \leftarrow \nabla\bar{\Psi}^*(\nabla\Psi(x_s) - \frac{1}{\sigma_s}\tilde{l}_s)$.
    \ENDFOR
\ENDFOR
\end{algorithmic}
\end{algorithm}

Compared to \texttt{Banker-SFTINF}, this algorithm uses log-barrier regularizers $\Psi(x)=-\sum_{i=1}^K \ln x_i$ \citep{abernethy2015fighting}. This regularizer tends to allocate more exploration probabilities to all arms \citep{agarwal2017corralling} and leads to various adaptation properties \citep{wei2018more} --- which enable us to derive a small-loss style \citep{neu2015first} regret bound. Besides the different regularizers, the main differences between \Cref{banker-sflbinf,banker-sftinf} are highlighted in {\color{blue}{blue}}, including different action scales (\Cref{line-alg2-sigma-before-max,line-alg2-sigma}) and one more skipping criterion (\Cref{line-alg2-skip}).

The change in the regularizer together with the more strict skipping criterion (\Cref{line-alg2-skip}) makes it possible to control each immediate cost term $\E[\sigma_t D_\Psi(x_t, \tilde{z}_t)]$ to by some quantity proportional to $1/\sigma_t$ even when the losses are negative. One can refer to \Cref{lemma-immedicate-cost-general-log-barrier} in the appendix for more details. The post-processing to the learning scales (\Cref{line-alg2-sigma}) is a technical modification matching up with \Cref{line-alg2-skip} to make the regret still of order $\O(\sqrt{K})$. At last, the meticulous $D_t$ in \Cref{line-alg2-D} allows us to additionally derive the small-loss style regret bound.

\subsection{Regret Decomposition}
The regret decomposition we will use here is slightly different from \Cref{eq-sbtinf-regret-decomposition-apdx} because $\Psi$ is now the log-barrier function, which explodes on the boundary of $\Delta^{[K]}$.

\begin{lemma}
Define
\begin{equation*}
    \Delta^{\prime[K]} \triangleq \{x \in \Delta^{[K]}:x_i \ge 1/T,\quad \forall i \in [K]\},
\end{equation*}
    then, \Cref{banker-sflbinf} guarantees
    \begin{equation}
        \mathfrak R_T \le \O(KL) + \underbrace{ B_T \sup_{y \in \triangle^{\prime[K]}}D_\Psi(y, x_0))}_{\text{Total investment}} + \underbrace{\sum_{t=1}^T \E[\sigma_t D_\Psi(x_t, \tilde z_t)] }_{\text{Immediate costs}} + \underbrace{\sum_{t=1}^T \E[\mathbbm 1_{\mathcal E_t} l_{t,A_t}]}_{\text{Skipping error}}
        \label{eq-sflbinf-regret-decomposition-apdx}
    \end{equation}
    where $\mathcal E_t$ denotes the event that the feedback of round $t$'s action is marked to ``skipped'', $\sigma_t$, $B_t$, $x_t$, $\tilde z_t$ are variables determined by \Cref{banker-sflbinf}.
\end{lemma}
\begin{proof}
Denote by $i^\ast$ the actual offline optimal arm, let $\tilde y$ be the $l_2$-projection of $\mathbf 1_{i^\ast}$ onto $\Delta^{\prime[K]}$, it is easy to see (for sufficiently large $T$)
\begin{equation*}
    \tilde y = \begin{cases}
        1 - \frac{K-1} T & \text{if } i = i^\ast \\
        \frac 1 T & \text{otherwise}.
    \end{cases}
\end{equation*}
Hence
\begin{align*}
    \mathfrak R_T & = \E[\sum_{t=1}^T \langle l_t, x_t - \mathbf 1_{i^\ast} \rangle] \\
    & =  \E[\sum_{t=1}^T \langle l_t, x_t - \tilde y \rangle] +  \E[\sum_{t=1}^T \langle l_t, \tilde y - \mathbf 1_{i^\ast} \rangle] \\
    & \le \E[\sup_{y\in\Delta^{\prime[K]}} \sum_{t=1}^T \langle l_t, x_t - \tilde y \rangle] + \E[ \sum_{t=1}^T \lVert l_t \rVert_\infty \cdot \lVert \tilde y - \mathbf 1_{i^\ast} \rVert_1] \\
    & \le E[\sup_{y\in\Delta^{\prime[K]}} \sum_{t=1}^T \langle l_t, x_t - \tilde y \rangle] + 2KT,
\end{align*}
the first term can then be analysed normally using \Cref{thm-banker-omd}.
\end{proof}

Below, we will control the three terms in \Cref{eq-sflbinf-regret-decomposition-apdx}. 

\subsection{Bound for Immediate Costs}

\texttt{Banker-SFLBINF} (\Cref{banker-sflbinf}) is designed for scale-free MAB problem instance with general (possibly negative) losses, where the argument when we prove \Cref{lemma-apdx-sftinf-single-immediate-cost} no longer works. When we are dealing with general losses, the importance sampling estimators $\tilde l_t$ may contain negative components, and we cannot simply upper bound the Hessian $\nabla^2\Psi^*(\theta_t)$ by $\nabla^2\Psi^*(\nabla x_t)$. To solve this issue, in \Cref{banker-sflbinf}, we instead use log-barrier $\Psi(x) = -\sum_{i=1}^K \ln(x_i)$ as the regularizer to utilize the following property:

\begin{lemma}\label{lemma-immedicate-cost-general-log-barrier}
For any $x_t$ in the interior of $\triangle^{[K]}$, $\sigma_t > 0$, $\tilde l_t \in \mathbb{R}^K$,  let $\Psi$ be the log-barrier function, $\tilde z_t=\nabla \Psi^\ast(\nabla \Psi(x_t)-\frac{1}{\sigma_t}\tilde l_t)$, if for all $i\in [K]$ we have $\tilde z_{t,i} \le 2 x_{t,i}$, then
\begin{equation*}
    \sigma_tD_\Psi(x_t,\tilde z_t)\le 2\sigma_t^{-1} \sum_{i=1}^K x_{t,i}^2 \tilde l_{t,i}^2.
\end{equation*}
In particular, if $A_t$ is an independent sample from $[K]$ according to $x_t$, $\hat l_t \in \mathbb R^{K}$, and $\tilde l_t = \frac {\hat l_{t,A_t}}{x_{t,A_t}}\mathbf 1_{A_t}$ $\tilde l_t$ is an importance sampling estimator determined by $\hat l_t$ and $A_t$, we have
\begin{equation*}
    \sigma_tD_\Psi(x_t,\tilde z_t)\le 2\sigma_t^{-1} l_{t,A_t}^2.
\end{equation*}
\end{lemma}

\begin{proof}
Similar to the proof of \Cref{lemma-apdx-sftinf-single-immediate-cost}, we can write
\begin{equation}
    \sigma_t D_\Psi(x_t, \tilde z_t) = \frac{1}{2}\sigma_t^{-1} \left\Vert \tilde l_t \right\Vert^2_{\nabla^2\Psi^*(w_t)} = \frac{1}{2}\sigma_t^{-1} \left\Vert \tilde l_t \right\Vert^2_{\nabla^2\Psi(\nabla\Psi^*(w_t))^{-1}} \label{eq-proof-lemma11-matrix-norm}
\end{equation}
where $w_t = \nabla\Psi(x_t) - \frac \theta {\sigma_t}\tilde l_t$ for some $\theta \in (0,1)$. When $\Psi$ is the log-barrier function, we have $\nabla\Psi(x) = (-1/x_1, -1/x_2,\ldots, -1/x_K)^T$, $\nabla\Psi^*(\theta) = (-1/\theta_1, -1/\theta_2,\ldots, -1/\theta_K)^{\mathsf T}$. The monotonicity of each coordinate of $\nabla\Psi^*$ implies that for any $i\in [K]$ we have \begin{equation*}
    \min\{\tilde z_{t,i}, x_{t,i}\} \le \nabla\Psi^*(w_t)_i \le \max\{\tilde z_{t,i}, x_{t,i}\}.
\end{equation*}

The condition $\tilde z_{t,i} \le 2 x_{t,i}$ for all $i\in [K]$ implies that $\nabla\Psi^*(w_t)_i \le 2x_{t,i}$ for all $i$. Plugging this upper bound and $\nabla^2\Psi(x) = \operatorname{\mathrm{diag}}(x_1^{-2},x_2^{-2},\ldots, x_K^{-2})$ into (\ref{eq-proof-lemma11-matrix-norm}), we get
\begin{equation*}
    \sigma_t D_\Psi(x_t, \tilde z_t) \le 2\sigma_t^{-1}\left\Vert \tilde l_t \right\Vert^2_{\nabla^2\Psi(x_t)^{-1}} = 2 \sigma_t^{-1} \sum_{i=1}^K x_{t,i}^2 \tilde l_{t,i}^2.
\end{equation*}
\end{proof}

Intuitively, Lemma~\ref{lemma-immedicate-cost-general-log-barrier} suggests that it remains safe to apply single-step OMD regret upper bounds when the intermediate result $\tilde z_t$ does not get too large, which is automatically guaranteed when the actual suffered loss $l_{t,A_t}$ is not too large compared with the action scale $\sigma_t$. This fact is formally stated as follows.

\begin{lemma}
\label{lemma-immedicate-cost-scale}
For $l_t, x_t, A_t, \tilde z_t$ discussed in \Cref{lemma-immedicate-cost-general-log-barrier}, if $l_{t,A_t} \ge -\frac 1 2 \sigma_t$, then $\tilde z_{t, A_t} \le 2 x_{t, A_t}$.
\end{lemma}

\begin{proof}
It suffice to investigate the mirror map
\begin{equation*}
\nabla\Psi(x) = (-1/x_1, -1/x_2,\ldots, -1/x_K)^{\mathsf T}
\end{equation*}
and 
\begin{equation*}
    \nabla\Psi^*(\theta) = (-1/\theta_1, -1/\theta_2,\ldots, -1/\theta_K)^{\mathsf T}.
\end{equation*}
Now $\tilde l_t$ is an importance sampling estimator, hence only the $A_t$-th coordinate can be non-zero and $\tilde z_{t,i} = x_{t,i}$ for all $i\ne A_t$.

As for the $A_t$-th coordinate, we have
\begin{align*}
    \tilde z_{t,A_t} & = \nabla\Psi^*\left (\nabla\Psi(x_t) - \frac 1 {\sigma_t} \tilde l_t\right )_{A_t} \\
    & = - \left(-x_{t,A_t}^{-1} - \frac {l_{t,A_t}} {\sigma_t x_{t,A_t}}\right)^{-1} \\
    & = x_{t,A_t}\cdot \left(1 + \frac {l_{t,A_t}}{\sigma_t}\right)^{-1},
\end{align*}

it is then easy to see that we have $\tilde z_{t,A_t} \le 2 x_{t,A_t}$ whenever $l_{t,A_t} \ge -\frac 1 2 \sigma_t$.
\end{proof}

Recall that in \Cref{banker-sflbinf} we slightly modified the skipping criteria (\Cref{line-alg2-skip}) to guarantee the effective \textit{pre-importance-sampling} $l_{t,A_t}$ fed into \texttt{Banker-OMD} no less than $-\frac 1 2 \sigma_t$, thus by \Cref{lemma-immedicate-cost-scale}, we can apply \Cref{lemma-immedicate-cost-general-log-barrier} to all summands in the immediate costs term in \Cref{eq-sbtinf-regret-decomposition-apdx}. Specifically, we have for all $1\le t\le T$,
\begin{equation*}
    \sigma_t D_\Psi(x_t, \tilde z_t) \le \mathbbm 1_{\neg\mathcal E_t}2\sigma_t^{-1}l_{t,A_t}^2. 
\end{equation*}

Define $\tilde D_t$ to be a quantity similar to $D_t$ used in the algorithm by
\begin{equation*}
    \tilde D_t \triangleq \sum_{s\le t:\neg\mathcal E_t} (\mathfrak d_s + 1) l_{s,A_s}^2.
\end{equation*}
Recall that the actual $D_t$ used in \Cref{banker-sflbinf} is
\begin{equation*}
    D_t \triangleq \sum_{s\le t:a_s = \text{``missing'' at the beginning of time }t} (\mathfrak d_s + 1)\hat L_s^2 + \sum_{s\in \{1,\ldots,t\}:a_s = \text{``arrived'' at }t} (\mathfrak d_s + 1) l_{s,A_s}^2
\end{equation*}
Compared to $D_t$, $\tilde D_t$ does not take into account rounds that were not skipped at the beginning of time $t$ but are skipped some time later. Also, for an unskipped round $s$, it contributes $(\mathfrak d_s + 1)l_{s,A_s}^2$ to $\tilde D_t$, but $l_{s,A_s}^2\hat L_s^2$ to $D_t$. It is thus guaranteed that $\tilde D_t \le D_t$.

We further define
\begin{equation*}
    \tilde \sigma_t=\left ((\mathfrak d_t+1)\sqrt\frac{\ln(3 + \tilde D_t/\hat L_t^2)}{{3+\tilde D_t}}\sqrt{K\ln T} \right)^{-1},
\end{equation*}
which is basically $\sigma_t$ with $D_t$ replaced by $\tilde D_t$. Then $\sigma_t \ge \tilde \sigma_t$, again by the monotonicity of $(3+x)/\ln(3+x/a)$. Then we can write

\begin{align}
    \sum_{t=1}^T \sigma_t D_\Psi(x_t, \tilde z_t) & \le 2\sum_{t=1}^T \mathbbm 1_{\neg\mathcal E_t}\sigma_t^{-1}l_{t,A_t}^2 \nonumber \\
    & = 2\sum_{t\le T: \neg \mathcal E_t} \sigma_t^{-1}l_{t,A_t}^2 \nonumber \\
    & \le 2\sum_{t\le T: \neg \mathcal E_t} \tilde \sigma_t^{-1}l_{t,A_t}^2 \nonumber \\
    & = 2\sqrt{K\ln T}\sum_{t\le T:\neg\mathcal E_t}\sqrt{\ln(3 + \tilde D_t / \hat L_t^2)}\cdot \frac {(\mathfrak d_t + 1)l_{t,A_t}^2} {\sqrt{3 + \tilde D_t}} \nonumber \\
    & \stackrel{(a)}{\le} 2\sqrt{K\ln T}\cdot \O(\sqrt{\log (D+T)})\sum_{t\le T:\neg\mathcal E_t} \frac {(\mathfrak d_t + 1)l_{t,A_t}^2} {\sqrt{3 + \tilde D_t}} \nonumber \\
    & \stackrel{(b)}{\le} 2\sqrt{K\ln T}\cdot \O(\sqrt{\log (D+T)}) \cdot \O(\sqrt{ 1 + \tilde D_T}) \label{eq-sflbinf-immediate-costs-last-step} \\
    & \le \O(L\sqrt{K(D+T)\log T\log(T+D)}) \nonumber \\
    & \le \O(L\sqrt{K(D+T)}\log T) \nonumber
\end{align}
where in step $(a)$ we make use of the fact that $\tilde D_t / \hat L_t^2 \le t + \sum_{s\le t}\mathfrak d_s$ and then universally bound all $\tilde D_t / \hat L_t^2$ by $D+T$; in step $(b)$ we utilize the fact that $D_t$ is the cumulative sum of $1_{\neg\mathcal E_s}(\mathfrak d_s + 1)l_{s,A_s}^2$, which is just the numerator of each summand, hence we can apply the summation lemma \Cref{lem:summation lemma}.

Let $\tilde {\mathfrak L}^2_T \triangleq \sum_{t=1}^T (\mathfrak d_t + 1) l_{t,A_t}^2$ denote the cumulative actually suffered square loss, weighted by the delay backlog size, we can see $\tilde D_T \le \tilde {\mathfrak L}^2_T$ since $\tilde D_T$ only takes unskipped time slots into account. Plugging $\tilde D_T \le \tilde {\mathfrak L}^2_T$ into \Cref{eq-sflbinf-immediate-costs-last-step}, then taking expectation on both sides, finally noticing that square root is a concave operation, we get
\begin{equation*}
    \E\left[\sum_{t=1}^T \sigma_t D_\Psi(x_t, \tilde z_t)\right] \le \O(\sqrt{K(1 + \E[\tilde {\mathfrak L}^2_T])}\log T).
\end{equation*}

\subsection{Bound for Total Investment Term}

To bound the total investment term, we first identify the order of the factor $\sup_{y \in \triangle^{\prime[K]}}D_\Psi(y, x_0))$:

\begin{lemma}
When $\Psi$ is the log-barrier function, we have $D_\Psi(y, x_0) \le K\ln T + K = \O(K\log T)$ for all $y \in \Delta^{\prime[K]}$.
\end{lemma}
\begin{proof}
Notice that $\Delta^{\prime[K]}$ is a compact convex subset of $\mathbb R^K_+$ (actually, it is a polyhedron), $\Psi$ is a convex function over $\mathbb R^K_+$, the maximum value of $\Psi$ must achieve on the boundary of $\Delta^{\prime[K]}$. Due to the symmetry of coordinates of $x_0$, it suffices to verify the bound for all vertices $y$ from $\Delta^{\prime[K]}$, and now we have
\begin{align*}
    D_\Psi(y,x_0) & = (K-1)\left(-\ln \frac 1 T + \ln \frac 1 K + K\left (\frac 1 T - \frac 1 K\right )\right) \\
    & \quad + \left(-\ln\left (1 - \frac {K-1} T\right ) + \ln \frac 1 K +K\left (1 - \frac {K
    -1} T - \frac 1 K \right ) \right) \\
    & \le (K-1) \ln T - \ln \left (1 - \frac{K-1} T\right ) + K \\
    & \le K \ln T + K.
\end{align*}
\end{proof}

It then suffices to bound the leading coefficient $B_T$. 

The choice of $\sigma_t$ in Algorithm~\ref{banker-sflbinf} (Line~\ref{line-alg2-sigma-before-max}, Line~\ref{line-alg2-sigma}) satisfies
\begin{equation*}
    \sigma_t \le \left [(\mathfrak d_t+1)\sqrt\frac{\ln(3 + D_t/\hat L_t^2)}{{3+D_t}}\sqrt{K\ln T}\right ]^{-1} + \mathbbm 1_{\{\mathfrak d_t \le \sqrt \frac {\mathfrak D_t} K\}}\cdot 2\hat L_t.
\end{equation*}
By Lemma \ref{lemma-banker-omd-B}, there exists some $t_0\le T$ such that
\begin{equation}\label{eq:total investment bound when non-negative}
B_T=B_{t_0}=\sigma_{t_0}+\sum_{s=1}^{t_0-1}\mathbb I[s+d_s\ge t_0]\sigma_s.
\end{equation}

Let $t_1<t_2<\cdots<t_m$ be the time slots whose feedback has not arrived at time slot $t_0$, we will have $m=\mathcal O(\sqrt D)$. Furthermore, at time slot $t_i$, we must have $\mathfrak d_{t_i}\ge i$ as the feedback of $t_1,t_2,\cdots,t_i$ are all absent, which can be used to bound the number of $i$'s satisfying $\mathfrak d_{t_i} \le \sqrt \frac {\mathfrak D_{t_i}} K$. Therefore, $B_T$ can be further bounded by
\begin{align}
    B_T & = \sum_{i=0}^m \sigma_{t_i} \nonumber \\
     & \le \sum_{i=0}^m \left\{ \left [(\mathfrak d_{t_i}+1)\sqrt\frac{\ln(3 + D_{t_i}/\hat L_{t_i}^2)}{{3+D_{t_i}}}\sqrt{K\ln T}\right ]^{-1} + \mathbbm 1_{\{\mathfrak d_{t_i} \le \sqrt \frac {\mathfrak D_{t_i}} K\}}\cdot 2\hat L_{t_i} \right\} \nonumber \\
     & \stackrel{(a)}{\le} \sum_{i=0}^m \left [(\mathfrak d_{t_i}+1)\sqrt\frac{\ln(3 + D_{t_i}/\hat L_{t_i}^2)}{{3+D_{t_i}}}\sqrt{K\ln T}\right ]^{-1} + 4\left(\sqrt \frac D K + 1\right)L \nonumber \\
     & = \mathcal O\left(\sqrt \frac D K L\right) + \frac 1 {\sqrt{K\ln T}} \sum_{i=0}^m \frac 1 {\mathfrak d_{t_i}+1}\sqrt\frac{{3+D_{t_i}}}{\ln(3 + D_{t_i}/\hat L_{t_i}^2)}  \label{eq-BT-argument-apdx} \\
     & \le \mathcal O\left(\sqrt \frac D K L\right) + \frac {2L} {\sqrt{K\ln T}} \sum_{i=0}^m \frac 1 {\mathfrak d_{t_i}+1}\sqrt\frac{{3+D_{t_i}/\hat L_{t_i}^2}}{\ln(3 + D_{t_i}/\hat L_{t_i}^2)} \nonumber \\
     & \stackrel{(b)}{\le} \mathcal O\left(\sqrt \frac D K L\right) + \frac {2L} {\sqrt{K\ln T}}\cdot \sqrt \frac {3 + D + T} {\ln(3 + D + T)} \sum_{i=0}^m \frac 1 {\mathfrak d_{t_i} + 1} \nonumber \\
     & \le \mathcal O\left(\sqrt \frac D K L\right) + \frac {2L} {\sqrt{K\ln T}}\cdot \sqrt \frac {3 + D + T} {\ln(3 + D + T)} \cdot \mathcal O(\log D) \nonumber \\
     & = \mathcal O\left(\sqrt \frac D K L + \sqrt\frac{(D+T)\log D}{K\log T}L\right) \nonumber
\end{align}
where in step $(a)$, we simply bound $\hat L_{t,i}$ by $2L$, and utilizing $\mathfrak d_{t_i} \ge i$ for all $i\ge 1$ to write $\sum_{i=0}^m \mathbbm 1_{\{\mathfrak d_{t_i} \le \sqrt \frac {\mathfrak D_{t_i}} K\}} \le \sum_{i=0}^m \mathbbm 1_{\{i \le \sqrt \frac {\mathfrak D_{t_i}} K\}} \le \sum_{0=1}^m \mathbbm 1_{\{i \le \sqrt \frac D K\}} \le \sqrt \frac D K + 1$; $(b)$ uses the monotonicity of $(3+x) / ln(3 + x)$ and the fact that $D_t / \hat L_t^2 \le t + \sum_{s=1}^t \mathfrak d_s$.

In order to get a bound in $\tilde {\mathfrak L}^2_T = \sum_{t=1}^T (\mathfrak d_t + 1) l_{t,A_t}^2$, similarly define $\tilde {\mathfrak L}^2_t \triangleq \sum_{s=1}^t (\mathfrak d_s + 1) l_{s,i_s}^2$ and define $\tilde m \triangleq \max_{1\le t\le T}\mathfrak d_t$. We must have $\tilde m = \O(\sqrt D)$ since $D \ge \binom {m + 1} 2$. Recall that in \Cref{banker-sflbinf} the value of $D_t$ we pick is

\begin{align*}
    D_t & = \sum_{s\le t:a_s = \text{``missing'' at }t} (\mathfrak d_s + 1)\hat L_s^2 + \sum_{s\le t:a_s = \text{``arrived'' before }t} (\mathfrak d_s + 1) l_{s,A_s}^2 \\
    & \le \sum_{s=1}^t (\mathfrak d_s + 1) l_{s,A_s}^2 + \hat L_t^2\sum_{s\le t:a_s = \text{``missing'' at }t} (\mathfrak d_s + 1) \\
    & \stackrel{(a)}{\le} \tilde {\mathfrak L}^2_t + \hat L_t^2(\tilde m + 1)^2 \\
    & \le \tilde {\mathfrak L}^2_t + 100 D\hat L_t^2
\end{align*}
where $(a)$ holds because in the sum $\sum_{s\le t:a_s = \text{``missing'' at }t} (\mathfrak d_s + 1)$, the number of summands and the value of each summand are both bounded by $\tilde m + 1$.

Leveraging this upper bound for $D_t$'s, we can continue from \Cref{eq-BT-argument-apdx} to obtain an upper bound in $\tilde {\mathfrak L}^2_T$:
\begin{align*}
    B_T & \le \mathcal O\left(\sqrt \frac D K L\right) + \frac 1 {\sqrt{K\ln T}} \sum_{i=0}^m \frac 1 {\mathfrak d_{t_i}+1}\sqrt\frac{{3+D_{t_i}}}{\ln(3 + D_{t_i}/\hat L_{t_i}^2)} \\
    & \le \mathcal O\left(\sqrt \frac D K L\right) + \frac 1 {\sqrt{K\ln T}} \sum_{i=0}^m \frac {\hat L_{t_i}} {\mathfrak d_{t_i}+1}\sqrt\frac{{3+D_{t_i} /\hat L_{t_i}^2}}{\ln(3 + D_{t_i}/\hat L_{t_i}^2)} \\
    & \stackrel{(a)}{\le} \mathcal O\left(\sqrt \frac D K L\right) + \frac 1 {\sqrt{K\ln T}} \sum_{i=0}^m \frac {\hat L_{t_i}} {\mathfrak d_{t_i}+1}\sqrt\frac{{3+\tilde {\mathfrak L}^2_{t_i} /\hat L_{t_i}^2} + 100D}{\ln(3 + \tilde {\mathfrak L}^2_{t_i}/\hat L_{t_i}^2 + 100D)} \\
    & \le \mathcal O\left(\sqrt \frac D K L\right) + \frac 1 {\sqrt{K\ln T}} \max_{1\le t\le T} \left\{ \hat L_t \sqrt\frac{{3+\tilde {\mathfrak L}^2_t /\hat L_t^2} + 100D}{\ln(3 + \tilde {\mathfrak L}^2_t/\hat L_t^2 + 100D)} \right\} \sum_{i=0}^m \frac 1 {\mathfrak d_{t_i} + 1} \\
    & \le \mathcal O\left(\sqrt \frac D K L\right) + \frac 1 {\sqrt{K\ln T}} \max_{1\le t\le T} \left\{ \hat L_t \sqrt{{3+\tilde {\mathfrak L}^2_t /\hat L_t^2} + 100D} \right\} \frac 1 {\sqrt{\ln(3+100D)}}\sum_{i=0}^m \frac 1 {\mathfrak d_{t_i} + 1} \\
    & = \mathcal O\left(\sqrt \frac D K L\right) + \frac 1 {\sqrt{K\ln T}} \max_{1\le t\le T} \left\{ \hat L_t \sqrt{{3+\tilde {\mathfrak L}^2_t /\hat L_t^2} + 100D} \right\} \frac {\mathcal O(\log D)} {\sqrt{\ln(3+100D)}} \\
    & \le \mathcal O\left(\sqrt \frac D K L\right) + \frac 1 {\sqrt{K\ln T}} \sqrt{\tilde {\mathfrak L}^2_T + (3 + 100D)\hat L_T^2} \frac {\mathcal O(\log D)} {\sqrt{\ln(3+100D)}} \\
    & = \mathcal O\left(\sqrt \frac D K L\right) + \mathcal O\left(\sqrt \frac {\log D} {K\log T} \sqrt{\tilde {\mathfrak L}^2_T + (3 + 100D)\hat L_T^2}\right) \\
    & \le \mathcal O\left(\sqrt \frac D K L\right) + \mathcal O\left(\sqrt\frac{\tilde{\mathfrak L}^2_T\log D}{K\log T} + \sqrt\frac {(1 + D)\log D}{K\log T}L\right)
\end{align*}
where step $(a)$ plugs in our $D_t$ bound into $\tilde {\mathfrak L}^2_t$. Taking expectation on both sides, we can see
\begin{equation*}
    \E[B_T] \le \O\left(\sqrt \frac D K L\right) + \O\left(\sqrt\frac{\E[\tilde{\mathfrak L}^2_T]\log D}{K\log T} + \sqrt\frac {(1 + D)\log D}{K\log T}L\right).
\end{equation*}

Combining the bounds for $\E[B_T]$ and $\sup D_\Psi(y,x_0)$, we can make the following conclusion:
\begin{lemma}
In \Cref{eq-sflbinf-regret-decomposition-apdx}, the total investment term when $y$ is restricted on $\Delta^{\prime[K]}$ is bounded by
\begin{align*}
    \sup_{y\in \Delta'_{[K-1]}} \E \left[B_T\cdot D_\Psi(y,x_0)\right] &= \mathcal O\left (\sqrt{ K D}\log T L + \sqrt{K(D+T)\log D\log T}L\right ),\\ \sup_{y\in \Delta'_{[K-1]}} \E \left[B_T\cdot D_\Psi(y,x_0)\right] &= \mathcal O\left(\sqrt {KD}\log T L + \sqrt{K\log D\log T \E[\tilde{\mathfrak L}^2_T]} + \sqrt{K(1+D)\log D\log T}L\right).
\end{align*}
\end{lemma}

\subsection{Bound for Skipping Error}

It finally remains to bound $\sum_{t=1}^T \E[\mathbbm 1_{\mathcal E_t}l_{t,A_t}]$, the skipping error term in \Cref{eq-sflbinf-regret-decomposition-apdx}. We claim the following result:
\begin{lemma}
In \Cref{banker-sflbinf}, the expected number of skipped time slots, namely the $\sum_{t=1}^T \E\left[ \mathbbm 1_{\mathcal E_t}\right]$ term in \Cref{eq-sflbinf-regret-decomposition-apdx}, is bounded by $\mathcal O(\sqrt D\log L+\sqrt {KD})$. Furthermore, the skipping regret is bounded by $\mathcal O((\sqrt D\log L+\sqrt {KD})L)$.
\end{lemma}
\begin{proof}
For any $1\le t \le T$, define the following two events
\begin{equation*}
    U_t \triangleq \left\{ \lvert l_{t,A_t} \rvert > \hat L_t \right\},
\end{equation*}
\begin{equation*}
    V_t \triangleq \left\{ \lvert l_{t,A_t} \rvert \le \hat L_t, l_{t,A_t} < -\frac 1 2 \sigma_t \right\}.
\end{equation*}
In other words, $U_t$ happens if and only if round $t$ is skipped by \Cref{banker-sflbinf} due to the skipping criterion inherited from \Cref{banker-sftinf}, and $V_t$ happens if and only if round $t$ is skipped \textit{solely} due to the new skipping criterion $l_{t,A_t} < -\frac 1 2 \sigma_t$. Hence our goal reduces to bound $\sum_{t=1}^T \E\left[ \mathbbm 1_{U_t}\right] + \sum_{t=1}^T \E\left[ \mathbbm 1_{V_t}\right]$, where the first sum is bounded by $\mathcal O(\sqrt{D}\log L)$ according the argument in \Cref{sec-sftinf-skipping-error-apdx}. Therefore, it suffices to bound $\sum_{t=1}^T \E\left[ \mathbbm 1_{V_t}\right]$.

Recall that we maintain experienced total delay $\mathfrak D_t$ in \Cref{banker-sflbinf} and we have $\mathfrak D_0 = 1$, $\mathfrak D_T = D + 1$. For any integer $i\ge 0$, define a stopping time
\begin{equation*}
    \tau_i \triangleq \inf\left\{ t \ge 0 : \mathfrak D_t \ge \frac {D}{2^i} \right\}.
\end{equation*}
Clearly, we have $\tau_i \le T$ and $\tau_0 \ge \tau_1 \ge \cdots$ almost surely holds. The idea is to bound the sum of $\mathbbm 1_{V_t}$ during any two successive stopping times $\tau_i$ and $\tau_{i-1}$. That is, to bound $\sum_{t=\tau_i}^{\tau_{i-1}-1} \mathbbm 1_{V_t}$ for each $i\ge 1$.

Notice that for a $t \ge \tau_i$, the value of $\mathfrak D_t$ is at least $D/2^i$. If $V_t$ happens, then at time $t$, Line~\ref{line-alg2-sigma} of \Cref{banker-sflbinf} cannot be executed (otherwise we will have $\sigma_t \ge 2\hat L_t$ and $l_{t,A_t} < -\frac 1 2 \sigma_t \le -\hat L_t$, a contradiction), which means $\mathfrak d_t > \sqrt \frac {\mathfrak D_t} K$. Therefore conditioned on $V_t$ and $t > \tau_i$, we have $\mathfrak D_t - \mathfrak D_{t-1} = \mathfrak d_t > \sqrt \frac {\mathfrak D_t} K \ge \sqrt \frac D {2^i K}$, and $\sum_{t=\tau_i}^{\tau_{i-1}-1} \mathbbm 1_{V_t}$ must be no more than $\frac D {2^{i-1}} / \sqrt \frac D {2^i K} = \sqrt{KD}2^{1-i/2}$. We can then conclude that
\begin{align*}
    \sum_{t=1}^T \mathbbm 1_{V_t} & = \sum_{t=\tau_0}^{T} \mathbbm 1_{V_t} + \sum_{i=1}^\infty \sum_{t=\tau_i}^{\tau_{i-1}-1} \mathbbm 1_{V_t} \\
    & \le \left .(D+1) \middle / \sqrt \frac D K\right . + \sum_{i=1}^\infty \sqrt{KD} 2^{-\frac i 2 + 1} = \mathcal O(\sqrt{KD}).
\end{align*}

Therefore, we have $\sum_{t=1}^T \E[\mathbbm 1_{\mathcal E_t}]=\sum_{t=1}^T \E[\mathbbm 1_{U_t}]+\sum_{t=1}^T \E[\mathbbm 1_{V_t}]=\mathcal O(\sqrt{D}\log L+\sqrt{KD})$. Furthermore, the total skipping regret is therefore no more than $L\sum_{t=1}^T \E[\mathbbm 1_{\mathcal E_t}]=\mathcal O(\sqrt{D}L\log L+\sqrt{KD}L)$.
\end{proof}

\section{Technical Details for \texttt{Banker-BOLO}}
\subsection{Algorithm Design}\label{sec:design of banker-bolo}

\begin{algorithm}[htb]
\caption{\texttt{Banker-BOLO} for Delayed Adversarial Linear Bandits}
\label{banker-bolo}
\begin{algorithmic}[1]
\REQUIRE{Number of dimension $n$; Time horizen length $T$; Legendre function $\Psi: \mathcal{C} \rightarrow \mathbb{R}$.}
\ENSURE{A sequence of actions $A_1, A_2,\ldots, A_T \in \mathcal{C}$.}

\STATE $\mathfrak{D}_0 \leftarrow 0$.
\FOR{$t=1,2,\ldots, T$}
    \STATE Set $a_t \leftarrow \mathrm{missing}$, $\mathfrak d_t\gets \sum_{s=1}^{t-1}\mathbbm 1_{\{a_s=\mathrm{missing}\}}$, and $\mathfrak{D}_t \leftarrow \mathfrak{D}_{t-1} + \mathfrak{d}_t$.
    \STATE Calculate $\sigma_t \leftarrow \max\left \{\left (\sqrt{\frac{\ln T}{nt}} + \mathfrak{d}_t\sqrt{\frac{\ln( \mathfrak{D}_t+1)\ln T}{n\mathfrak{D}_t}}\right )^{-1}, 8n\right \}$.
    \STATE Pick new action $x_t$ as indicated by \Cref{banker-omd}.
    \STATE Let $\{e_{t,1},\ldots,e_{t,n}\}$ and $\{\lambda_{t,1},\ldots,\lambda_{t,n}\}$ be the eigenvectors and eigenvalues of $\nabla^2\Psi(x_t)$.
    \STATE Sample $i_t$ uniformly from $[n]$ and $\varepsilon_t$ from a Rademacher random variable. 
    \STATE Play $A_t$ defined by $A_t=x_t + \varepsilon_t \lambda_{t,i_t}^{-\nicefrac 12}e_{t,i_t}$. \alglinelabel{line-bolo-sample} \COMMENT{sample on the Dikin ellipsoid}
    \FOR{upon each new feedback $(s, \hat{l}_s)$}
        \STATE $\tilde{l}_s \leftarrow \hat{l}_s\cdot n\varepsilon_s\lambda_{s,i_s}^{\nicefrac 12}\cdot e_{s,i_s}$. \COMMENT{loss estimator matching with the sampling scheme}
        \STATE Set $z_s \leftarrow \nabla\Psi^*(\nabla\Psi(x_s) - \frac{1}{\sigma_s}\tilde{l}_s)$ and $a_s \leftarrow \mathrm{arrived}$.
    \ENDFOR
\ENDFOR
\end{algorithmic}
\end{algorithm}

In the language of \texttt{Banker-OMD}, \texttt{Banker-BOLO} uses $x_0=\nabla\Psi^*(\mathbf{0})$ as the default investment option. As noticed by \citet{abernethy2008competing}, when $\Psi$ is $\O(n)$-self-concordant and $\sigma_t$'s are at least $8n$, this choice ensures $(\Psi, x_0)$ to be $(\O(n\log T), \O(n^2))$-regular under the sampling scheme over Dikin ellipsoids (\Cref{line-bolo-sample}). According to \Cref{thm-banker-omd-B}, we then pick the action scale as $\sigma_t = \max\left \{\left (\sqrt{\frac{\ln T}{nt}} + \mathfrak{d}_t\sqrt{\frac{\ln( \mathfrak{D}_t+1)\ln T}{n\mathfrak{D}_t}}\right )^{-1},8n\right \}$ and achieve $\Otil(n^{3/2}\sqrt T+n^2\sqrt D)$ regret.

In the following section, we give more rigorous justification on the applicability of \texttt{Banker-OMD} to the linear bandit problem, the $(\O(n \log T), \O(n^2))$-regularity of the regularizer, and the sampling scheme combination used by \texttt{Banker-BOLO}.

\subsection{Regret Analysis}
When we use a regularizer $\Psi$ that is a Legendre function natively defined on the convex action set $\mathcal{C}$, we have the following alternatives for \Cref{lemma-omd-basic} and \Cref{lemma-banker-convex-comb}:

\begin{lemma}
\label{lemma-omd-basic-C}
Let $\mathcal{C}$ be a convex subset of $\mathbb{R}^n$ with non-empty interior, then for any $\sigma > 0$, $x,y \in \mathop{int}(\mathcal{C})$, $l \in \mathbb{R}^n$ and Legendre function $\Psi: \mathcal{C} \ \rightarrow \mathbb{R}$, we have
\begin{equation*}
    \left\langle l, x - y \right\rangle \le \sigma D_\Psi(y, x) - \sigma D_\Psi(y, z) + \sigma D_\Psi(x, z) \label{eq-lemma-omd-basic-C}
\end{equation*}
where
\begin{equation}
    z = \mathop{\arg\min}_{x' \in \triangle^{[K]}} \left\langle l, x' \right\rangle + D_\Psi(x', x), \label{eq-lemma-omd-basic-z-def-C}
\end{equation}
or equivalently,
\begin{equation*}
    z = \nabla\Psi^*(\nabla\Psi(x) - \frac{1}{\sigma}l).
\end{equation*}
\end{lemma}

\begin{lemma}
\label{lemma-banker-convex-comb-C}
For any $m \ge 1$, $z_1,\ldots, z_m \in \mathop{int}(\mathcal{C})$, $\sigma_1,\ldots,\sigma_m > 0$ and Legendre function $\Psi: \mathcal{C} \rightarrow \mathbb{R}$, let $\sigma = \sum_{i=1}^m \sigma_i$ and
\begin{equation*}
    x = \nabla{\Psi}^*(\sum_{i=1}^m \frac{\sigma_i}{\sigma}\nabla\Psi(z_i)),
\end{equation*}
 we have 
\begin{equation*}
    \sigma D_\Psi(y, x) \le \sum_{i=1}^m \sigma_i D_\Psi(y, z_i)
\end{equation*}
for any $y \in \mathop{int}(\mathcal{C})$.
\end{lemma}

\Cref{lemma-omd-basic-C} and \Cref{lemma-banker-convex-comb-C} can be proved in almost the same way as \Cref{lemma-omd-basic} and \Cref{lemma-banker-convex-comb-C}. In fact, in the case of $\Psi$ natively defined on $\mathcal{C}$, the update rule \Cref{eq-lemma-omd-basic-z-def-C} already guarantees $l = \sigma (\nabla\Psi(x) - \nabla\Psi(y))$, we no longer need $\bar{\Psi}$, a ``constrained version'' of $\Psi$ as we did in the MAB case.

Follow the same reasoning in \Cref{sec-banker-omd}, we can see that \texttt{Banker-BOLO} has the following regret bound:
\begin{equation}
    \sum_{t=1}^T \langle \tilde{l}_t, x_t - y \rangle \le (\sum_{t=1}^T b_t) D_\Psi(y, \nabla\Psi^*(\mathbf{0})) + \sum_{i=1}^T \sigma_t D_\Psi(x_t, z_t) \
\end{equation}
for any $y \in \mathop{int}(\mathcal{C})$. \Cref{thm-banker-omd-B} continues to apply to $\sum_{t=1}^T b_t$ and gives an $\O(\sqrt{T} + \sqrt{D\log D})$ bound for this total investment coefficient. In order to derive the claimed regret bound in \Cref{thm-banker-bolo}, it remains to justify the following properties
\begin{itemize}
    \item $\tilde{l}_t$ is an unbiased estimate for the true loss vector $l_t$, i.e., $\E[\tilde{l}_t\mid \mathcal{F}_{t-1}] = l_t$;
    \item $\E[\sigma_t D_\Psi(x_t, z_t)]$ can be bounded by $\O(n^2 / \sigma_t)$;
    \item $D_\Psi(y, \nabla\Psi^*(\mathbf{0}))$ can be uniformly bounded by $\O(n \log T)$.
\end{itemize}
These properties are, therefore, all unrelated to the presence of feedback delays, and all have been proved in \citep{abernethy2008competing}. For the sake of completeness, we put the most important technical lemmas here.

\begin{definition}\label{def:self-concordant}
A self-concordant function $\Psi: \mathcal{C} \rightarrow \mathbb{R}$ is a $C^3$ convex function such that
\begin{equation*}
    \vert D^3\Psi(x)[h, h, h] \vert \le 2(D^2\Psi(x)[h, h])^{3/2}.
\end{equation*}
Here, the third-order differential is defined as
\begin{align*}
    D^3\Psi(x)[h, h, h] \triangleq  \frac{\partial^3}{\partial t_1 \partial t_2 \partial t_3}\mid_{t_1 = t_2 = t_3 = 0} \Psi(x + t_1 h_1 + t_2 h_2 + t_3 h_3).
\end{align*}

It is further called $\vartheta$-self-concordant if
\begin{align*}
    \vert D\Psi(x)[h] \vert \le \vartheta^{\nicefrac 12}[D^2\Psi(x)[h,h]]^{\nicefrac 12}.
\end{align*}
\end{definition}

\begin{definition}
If $\Psi$ is a self-concordant barrier over $\mathcal{C}$, for $x \in \mathop{int}(\mathcal{C})$ and $r>0$, define the open Dikin ellipsold of radius $r$ centered at $x$ as the set
\begin{equation*}
    W_r(x) \triangleq  \left\{ y \in \mathcal{C}: \Vert y - x \Vert_{\nabla^2\Psi(x)^{-1}} \le r \right\}.
\end{equation*}
\end{definition}

We have the following properties of Dikin ellipsoids (refer to \citep[Page 23]{nemirovski2004interior} for proofs):
\begin{lemma}
\label{lemma-dikin}
For any $x \in \mathop{int}(\mathcal{C})$, we have
\begin{enumerate}
    \item $W_1(x) \subseteq \mathcal{C}$;
    \item For any $\Vert h \Vert_{\nabla^2\Psi(x)^{-1}} < 1$, we have
    \begin{equation*}
        (1 - \Vert h \Vert_{\nabla^2\Psi(x)^{-1}})^2 \nabla\Psi^2(x) \preccurlyeq \nabla^2\Psi(x + h) \preccurlyeq (1 - \Vert h \Vert_{\nabla^2\Psi(x)^{-1}})^{-2} \nabla\Psi^2(x).
    \end{equation*}
\end{enumerate}
\end{lemma}

\Cref{lemma-dikin} asserts that any Dikin ellipsold of radius $1$ centered in the interior of $\mathcal{C}$ will be contained in $\mathcal{C}$. Recall that in \Cref{banker-bolo}, the new action $A_t$ is sampled from the surface of a unit-radius Dikin ellipsold centered at $x_t$ (\Cref{line-bolo-sample}), so $A_t$ is guaranteed to be a valid action in $\mathcal{C}$. \Cref{lemma-dikin} also states that, within a unit-radius Dikin ellipsold inside $\mathcal{C}$, the Hessians of $\Psi$ are ``almost proportional'' to the Hessian at the center of the ellipsold. This fact plays a crucial role in bounding immediate costs $\E[\sigma_t D_\Psi(x_t, z_t)]$. Prior that, we need another lemma.

\begin{lemma}
\label{lemma-z-close}
For any $1\le t\le T$, we have
\begin{equation*}
    z_t \in W_{\frac{4n}{\sigma_t}}(x_t).
\end{equation*}
\end{lemma}
\Cref{lemma-z-close} is Lemma 6 of \citep{abernethy2008competing}; refer to the original paper for a proof. It claims that the single-step OMD image $z_t$ is located within a Dikin ellipsold centered at $x_t$. Combining \Cref{lemma-dikin} and \Cref{lemma-z-close}, we can derive an upper bound for immediate costs.

\begin{lemma}
In \texttt{Banker-BOLO}, if $\sigma_t \ge 8n$, we have
\begin{equation*}
    \E[\sigma_t D_\Psi(x_t, z_t) \mid \mathcal{F}_{t-1}] \le \frac{2n^2}{\sigma_t}.
\end{equation*}
\end{lemma}
\begin{proof}
Similar to the proof of \Cref{lemma-mab-immediate}, we have
\begin{align*}
    \sigma_t D_\Psi(x_t, z_t) = \frac{\Vert \tilde{l}_t \Vert_{\nabla^2\Psi^*(\theta_t)}^2}{2\sigma_t} = \frac{\Vert \tilde{l}_t \Vert_{\nabla^2\Psi(w_t)^{-1}}^2}{2\sigma_t}
\end{align*}
where $\theta_t$ is some element inside the line segment connecting $\nabla\Psi(x_t) - \frac{1}{\sigma_t}\tilde{l}_t$ and $\nabla\Psi(x_t)$, $w_t = \nabla\Psi^*(\theta_t)$. According to \Cref{lemma-z-close}, $w_t$ is located in $W_{\frac{4n}{\sigma_t}}(x_t) \subseteq W_{\nicefrac 12}(x_t)$. We can thus apply the second result in \Cref{lemma-dikin} to get
\begin{equation*}
    \sigma_t D_\Psi(x_t, z_t)  = \frac{\Vert \tilde{l}_t \Vert_{\nabla^2\Psi(w_t)^{-1}}^2}{2\sigma_t} \le \frac{2\Vert \tilde{l}_t \Vert_{\nabla^2\Psi(x_t)^{-1}}^2}{\sigma_t}
\end{equation*}
hence
\begin{align*}
    \E[\sigma_t D_\Psi(x_t, z_t) \mid \mathcal{F}_{t-1}] & \le \frac{2}{\sigma_t} \E[ \Vert \tilde{l}_t \Vert_{\nabla^2\Psi(w_t)^{-1}}^2 \mid \mathcal{F}_{t-1}] \\
    & \le \frac{2}{\sigma_t} \E[ \Vert n\lambda_{t,i_t}^{\nicefrac 12}\cdot e_{t,i_t} \Vert_{\nabla^2\Psi(w_t)^{-1}}^2 \mid \mathcal{F}_{t-1}] \\
    & = \frac{2n^2}{\sigma_t}.
\end{align*}
\end{proof}

In order to derive an upper bound for the investment term $D_\Psi(y, \nabla\Psi^*(\mathbf{0}))$, we first introduce the following property of $\vartheta$-self-concordant barriers. The proof can be found in \citep[Page 34]{nesterov1994interior}.
\begin{lemma}
\label{lemma-self-cord-radius}
Define
\begin{equation*}
    \pi_{y}(x)\triangleq  \inf\{t \ge 0 : y + t^{-1}(x-y) \in \mathcal{C}\}.
\end{equation*}
If $\Psi$ is a $\vartheta$-self-concordant barrier on $\mathcal{C}$, then
for any $x,y \in \mathop{int}(\mathcal{C})$, we have
\begin{equation*}
    \Psi(y) - \Psi(x) \le \vartheta \ln \frac{1}{1- \pi_{x}(y)}.
\end{equation*}
\end{lemma}
\begin{corollary}
Define $\mathcal{C}_T \triangleq  \{y \in \mathcal{C}: \pi_{\nabla\Psi^*(\mathbf{0})}(y) \le 1 - 1/T\}$, if $\Psi$ is a $\vartheta$-self-concordant barrier on $\mathcal{C}$, we have 
\begin{equation*}
    \sup_{y \in \mathcal{C}_T} D_\Psi(y, \nabla\Psi^*(\mathbf{0})) \le \vartheta \ln T.
\end{equation*}
\end{corollary}

This uniform upper bound for $D_\Psi(y, \nabla\Psi^*(\mathbf{0}))$ over $\mathcal{C}_T$ is enough for us to design linear bandits algorithms. The reason is, for any $y \in \mathcal{C} \setminus \mathcal{C}_T$, the definition of $\mathcal{C}_T$ guarantees that there exists a $y' \in \mathcal{C}_T$ such that
\begin{equation*}
\sum_{t=1}^T \vert \langle l_t, y - y' \rangle \vert = \O(1)
\end{equation*}
uniformly holds. Therefore there is a only constant difference between $\mathfrak{R}_T$ and $\sup_{y \in \mathcal{C}_T} \E[\sum_{t=1}^T \langle l_t, x_t - y \rangle]$.

\end{document}